\documentclass[12pt,journal,onecolumn]{IEEEtran}
\usepackage{amsmath,amsfonts,amssymb}
\usepackage{algorithmic}
\usepackage[linesnumbered,ruled]{algorithm2e}
\usepackage[sort&compress, numbers]{natbib}
\usepackage{array}
\usepackage{textcomp}
\usepackage{stfloats}
\usepackage{url}
\usepackage{verbatim}
\usepackage{graphicx}
\usepackage{amsthm}
\usepackage{multirow}
\usepackage{optidef}
\usepackage{subfigure}
\usepackage{xcolor}
\allowdisplaybreaks

\DeclareMathOperator{\Var}{Var}
\DeclareMathOperator{\sign}{sign}
\DeclareMathOperator{\Lap}{Lap}

\DeclareMathOperator{\Clip}{Clip}
\newtheorem{thm}{Theorem}
\newtheorem{ass}{Assumption}
\newtheorem{defi}{Definition}
\newtheorem{lem}{Lemma}
\newtheorem{rmk}{Remark}
\newtheorem{prop}{Proposition}

\newcommand{\norm}[1]{\left\lVert#1\right\rVert}
\usepackage{etoolbox}
\makeatletter
\patchcmd{\@makecaption}
  {\scshape}
  {}
  {}
  {}
\makeatother

\begin{document}

\title{On Theoretical Limits of Learning with Label Differential Privacy}

\author{Puning Zhao, Chuan Ma, Li Shen, Shaowei Wang, Rongfei Fan\thanks{Puning Zhao and Li Shen are with School of Cyber Science and Technology, Shenzhen Campus of Sun Yat-sen University, Shenzhen 518107, China. Email: \{pnzhao1,mathshenli\}@gmail.com. Chuan Ma is with College of Computer Science, Chongqing University. Email: chuan.ma@cqu.edu.cn. Shaowei Wang is with School of Artificial Intelligence, Guangzhou University. Email: wangsw@gzhu.edu.cn. Rongfei Fan is with the School of Cyberspace Science and Technology, Beijing Institute of Technology. Email: : fanrongfei@bit.edu.cn (Corresponding author: Li Shen.)}}


\maketitle

\begin{abstract}
	Label differential privacy (DP) is designed for learning problems involving private labels and public features. While various methods have been proposed for learning under label DP, the theoretical limits remain largely unexplored. In this paper, we investigate the fundamental limits of learning with label DP in both local and central models for both classification and regression tasks, characterized by minimax convergence rates. We establish lower bounds by converting each task into a multiple hypothesis testing problem and bounding the test error. Additionally, we develop algorithms that yield matching upper bounds. Our results demonstrate that under label local DP (LDP), the risk has a significantly faster convergence rate than that under full LDP, i.e. protecting both features and labels, indicating the advantages of relaxing the DP definition to focus solely on labels. In contrast, under the label central DP (CDP), the risk is only reduced by a constant factor compared to full DP, indicating that the relaxation of CDP only has limited benefits on the performance.
\end{abstract}

\section{Introduction}
With the growing privacy concerns of personal data, differential privacy (DP) \citep{dwork2006calibrating} has emerged as a standard approach for privacy protection and has been applied in several high-tech giants \citep{erlingsson2014rappor,cormode2018privacy,apple,ding2017collecting,wang2019answering}. However, in real-world learning tasks, DP mechanisms often lead to significant performance degradation \citep{abadi2016deep,tramer2021differentially,bu2022scalable,de2022unlocking,wei2022dpis}. To mitigate this issue, researchers have recently proposed to relax the definition of DP to allow for performance improvements. In many practical applications, features tend to be far less sensitive than labels \citep{mcmahan2013ad,mcsherry2009differentially}. For instance, in medical machine learning problems, patients may wish to keep their diagnoses private, while demographic and clinical information used for prediction is comparatively less sensitive \citep{bussone2020trust,cao2025machine}. Inspired by this, Label DP has gained attention \cite{ghazi2021deep}, which treats features as public while ensuring privacy protection for labels only.

Recent years have witnessed significant progress in learning with label DP \citep{ghazi2021deep,ghazi2022regression,malek2021antipodes,esfandiari2022label,tang2022machine,zhao2024enhancing,ma2024locally}. While empirical results have shown consistent improvement, there is still a lack of understanding on the information-theoretic limits. An intriguing question arises: By relaxing DP definitions to focus solely on labels, how much accuracy improvement is information-theoretically achievable? Understanding this question in depth would provide valuable insights into whether sacrificing feature privacy is a worthwhile trade-off for enhanced accuracy.

Similar to the full DP, depending on whether the data analyzer is trusted, label DP can be defined on either central or local models. Under the local model, the label of each sample is privatized before uploading to the data analyzer. Under the central model, the data analyzer gathers accurate labels, and only randomizes the trained model parameters. Under label local DP (LDP), the analysis of theoretical limit is relatively easier, since we can use many existing tools for minimax analysis developed for full LDP \citep{duchi2013local,duchi2018minimax}. Compared with the full LDP setting, the additional complexity of analyzing label LDP is that the bounds of information-theoretic quantities of mixtures of public (i.e. feature) and private (i.e. label) components have not been analyzed before. Under label central differential privacy (CDP), the analysis becomes significantly more challenging. Even for full CDP, the understanding of information-theoretic limits remains incomplete. Existing works have focused on analyzing the minimax error in stochastic optimization \citep{abadi2016deep,chaudhuri2011sample,bassily2014private,bassily2019private,bassily2021differentially,feldman2020private,wang2020differentially,das2023beyond,asi2021private,zhao2024differential,ghazi2024convex}, whose building block is called the packing method \cite{hardt2010geometry}. Nevertheless, the packing method is only applicable to parameter spaces with fixed dimensionality, and is therefore unable to characterize the complexity of nonparametric learning problems. Therefore, even when model weights reach their optimal values, the learner's risk may still fall short of optimality due to the limitations of model structures. Consequently, new analytical techniques are required to account for both optimization and approximation risks in the overall risk assessment. 
\begin{table}[t]
\begin{center}
		\small
		\begin{tabular}{c|c|c|c}
			\hline
			Task & Type & Local model & Central model\\
			\hline 
			\multirow{3}{*}{Classification} & Label & $\tilde{O}((N\epsilon^2)^{-\frac{\beta(\gamma+1)}{2\beta + d}})$& $\tilde{O}\left(N^{-\frac{\beta(\gamma+1)}{2\beta+d}}+(\epsilon N)^{-\frac{\beta(\gamma+1)}{\beta+d}}\right)$\\
			\cline{2-4}
			&Full &$O((N\epsilon^2)^{-\frac{\beta(\gamma+1)}{2\beta+2d}})$ \cite{berrett2019classification} & $\tilde{O}\left(N^{-\frac{\beta(\gamma+1)}{2\beta+d}}+(\epsilon N)^{-\frac{\beta(\gamma+1)}{\beta+d}}\right)$\\
			\cline{2-4}
			&Non-private & \multicolumn{2}{|c}{$O(N^{-\frac{\beta(\gamma+1)}{2\beta+d}})$ \cite{chaudhuri2014rates}}\\
			\hline
			\multirow{3}{*}{Regression} &Label & $\tilde{O}\left((N\epsilon^2)^{-\frac{2\beta}{d+2\beta}} \right)$ & $O\left(N^{-\frac{2\beta}{2\beta+d}}+(\epsilon N)^{-\frac{2\beta}{d+\beta}}\right)$\\
			\cline{2-4}
			&Full &$O((N\epsilon^2)^{-\frac{\beta}{\beta+d}})$ \cite{berrett2021strongly}&$O\left(N^{-\frac{2\beta}{2\beta+d}}+(\epsilon N)^{-\frac{2\beta}{d+\beta}}\right)$\\
			\cline{2-4}
			Bounded noise & Non-private & \multicolumn{2}{|c}{$O(N^{-\frac{2\beta}{2\beta+d}})$ \cite{tsybakov2009introduction,zhao2019minimax}} \\
			\hline
			\multirow{3}{*}{Regression} & Label &$O\left((N\epsilon^2)^{-\frac{2\beta(p-1)}{2p\beta+d(p-1)}}\vee N^{-\frac{2\beta}{2\beta+d}}\right)$ &$O\left(N^{-\frac{2\beta}{2\beta+d}}+(\epsilon N)^{-\frac{2\beta(p-1)}{p\beta+d(p-1)}}\right)$\\
			\cline{2-4}
			& Full &$O((N\epsilon^2)^{-\frac{\beta(p-1)}{p\beta+d(p-1)}})$ \cite{berrett2021strongly} &$O\left(N^{-\frac{2\beta}{2\beta+d}}+(\epsilon N)^{-\frac{2\beta(p-1)}{p\beta+d(p-1)}}\right)$ \\
			\cline{2-4}
			Unbounded noise &Non-private & \multicolumn{2}{|c}{$O(N^{-\frac{2\beta}{2\beta+d}})$ \cite{tsybakov2009introduction,zhao2019minimax}} \\
			\hline
		\end{tabular}
	\end{center}
	\caption{Minimax rate of convergence under label differential privacy for $d$ dimensional feature spaces. under the $\beta$-H{\"o}lder smoothness. For the classification problem, there is an additional $\gamma$-Tsybakov margin condition. For regression with unbounded noise, the label is assumed to have finite $p$-th order moments.}\label{tab}
\end{table}

In this paper, we investigate the theoretical limits of classification and regression problems with label DP under both central and local models. For each problem, we derive the information-theoretic minimax lower bound of the risk function over a wide class of distributions satisfying the $\beta$-H{\"o}lder smoothness and the $\gamma$-Tsybakov margin assumption \citep{tsybakov2009introduction} (see Assumption \ref{ass:multiclass} for details). The results are shown in Table \ref{tab}, in which we compare the convergence rates of risks between the local and central model, as well as label DP and full DP. We also show non-private rates. The comparison between these bounds reveals that under the local model, for various tasks, the minimax optimal convergence rates under label DP are significantly faster than that under full DP, indicating that the accuracy can be significantly improved by weakening the DP definition to only labels. On the contrary, under the central model, the minimax rates under label DP are the same as those under full DP. The reduction of risks by weakening the DP definition to only labels is only up to a constant factor. Intuitively, the local randomization of high-dimensional feature vectors introduces much more serious impact than under central model \cite{duchi2018minimax}. As a result, relaxing the DP definition to only labels provides a significantly greater accuracy gain in the local model than in the central model.

To construct lower bounds, we reformulate each task as a multiple hypothesis testing problem. Under the local model, we obtain a bound on the Kullback-Leibler (KL) divergence of joint distributions involving both private and public components, which is tighter than the existing bound for fully private variables. Furthermore, under the central model, to address the aforementioned limitations of prior analyses on stochastic optimization, we develop new lower bounds on the minimum testing error for each hypothesis pair, without imposing any parametric assumptions. This enables us to establish lower bounds for the overall risk across all models, rather than limiting our analysis to the optimization risk associated with a fixed model. Following the derivation of minimax lower bounds, we also propose algorithms and provide upper bounds. For the classification problem, we use a mechanism similar to Rappor \citep{erlingsson2014rappor} for the local model and an exponential mechanism for the central model. For the regression problem, we propose a nearest neighbor approach for the local model and a locally averaging technique with cube splitting for the central model. The upper bounds match the lower bounds, thereby demonstrating the tightness of our analysis.

The findings of this paper are summarized as follows.
\begin{itemize}
	\item \textbf{Comparison of label DP with full DP}. Under the local model, the risk of label LDP converges significantly faster than full LDP, indicating clear performance improvements of relaxation of DP to only labels. In contrast, under the central model, the convergence rate remains the same, indicating that the risk of label CDP is only reduced by a constant factor compared to full CDP.
	
	\item \textbf{Comparison of label DP with non-private rates.} Under the local model, compared with non-private rates, for classification and regression with bounded noise, the risk is enlarged by a factor of $\epsilon$, while the convergence rate with respect to $N$ remains the same. For regression with unbounded noise, privacy protection requires larger additional noise, resulting in a slower convergence rate. In contrast, under the central model, DP mechanisms introduce much smaller additional risk, which converges faster than the non-private rate. 
\end{itemize}

\section{Related Work}\label{sec:related}
\textbf{Label DP.} Under the local model, labels are randomized before training. For the classification problem, the simplest method is randomized response \citep{warner1965randomized}. An important improvement is proposed in \citep{ghazi2021deep}, called RRWithPrior, which incorporates prior distribution. \citep{malek2021antipodes} proposes ALIBI, which further improves randomized response by generating soft labels through Bayesian inference. \citep{zhao2024enhancing} proposed a vector approximation approach to significantly improve the accuracy of classification under the label LDP with a large number of classes. There are also several methods for regression under label LDP \citep{ghazi2022regression,badanidiyuru2023optimal}. Under the central model, \citep{esfandiari2022label} proposes a clustering approach. \citep{malek2021antipodes} applies private aggregation of teacher ensembles (PATE) \citep{papernot2017semi} into label CDP setting. This approach is further improved in \citep{tang2022machine}.

\textbf{Minimax analysis for public data.} Minimax theory provides a rigorous framework for the best possible performance of an algorithm given some assumptions. Classical methods include Le Cam \citep{lecam1973convergence}, Fano \citep{verdu1994generalizing} and Assouad \citep{assouad1983deux}. Using these methods, minimax lower bounds have been widely established for both classification and regression problems \citep{yang1999minimax,yang1999minimax2,audibert2007fast,tsybakov2009introduction,chaudhuri2014rates,yang2015minimax,doring2018rate,gadat2016classification,zhao2019minimax,zhao2021minimax}. If the feature vector has bounded support, then the minimax rate of classification and regression are $O(N^{-\frac{\beta(\gamma+1)}{2\beta+d}})$ and $O(N^{-\frac{2\beta}{2\beta+d}})$, respectively.

\textbf{Minimax analysis for private data.} Under the local model, \citep{kasiviswanathan2011can} finds the relation between local DP and statistical query. \citep{duchi2013local} and \citep{duchi2018minimax} develop the variants of Le Cam, Fano, and Assouad's method under local DP. Lower bounds are then established for various statistical problems, such as mean estimation \citep{li2023robustness,feldman2020private,duchi2019lower,huang2021instance}, hypothesis testing \cite{gopi2020locally}, classification \citep{berrett2019classification}, and regression \citep{berrett2021strongly}. Under the central model, for pure DP, the standard approach is the packing method \citep{hardt2010geometry}, which is then used in hypothesis testing \citep{bun2019private}, mean estimation \citep{narayanan2023better,kamath2020private}, and learning of distributions \citep{kamath2019privately,alabi2023privately,arbas2023polynomial}. There are also several works on approximate DP, such as \citep{bun2014fingerprinting,kamath2022new}. 

To the best of our knowledge, our work is the first attempt to conduct a systematic study the minimax rates under label DP. We provide a detailed comparison of the convergence rates of nonparametric classification and regression under label DP with that under full DP as well as the non-private rates, under both central and local models.


\section{Preliminaries}\label{sec:prelim}

In this section, we provide some necessary definitions, background information, and notations.

\subsection{Label DP}
To begin with, we review the definition of DP. Suppose the dataset consists of $N$ samples $(\mathbf{x}_i, y_i)$, $i=1,\ldots, N$, in which $\mathbf{x}_i\in \mathcal{X}\subset \mathbb{R}^d$ is the feature vector, while $y_i\in \mathcal{Y}\subset \mathbb{R}$ is the label.

\begin{defi}\label{def:dp}
	(Differential Privacy (DP) \citep{dwork2006calibrating}) Let $\epsilon\geq 0$. A randomized function $\mathcal{A}:(\mathcal{X}, \mathcal{Y})^N\rightarrow \Theta$ is $\epsilon$-DP if for any two adjacent datasets $D, D'\in (\mathcal{X}, \mathcal{Y})^N$ and any $S\subseteq \Theta$,
	\begin{eqnarray}
		\text{P}(\mathcal{A}(D)\in S)\leq e^\epsilon \text{P}(\mathcal{A}(D')\in S),
		\label{eq:dp}
	\end{eqnarray}
	in which $D$ and $D'$ are adjacent if they differ only on a single sample, including both the feature vector and the label.
\end{defi}

In machine learning tasks, the output of $\mathcal{A}$ is the model parameters, while the input is the training dataset. Definition \ref{def:dp} requires that both features and labels are privatized. Consider that in some applications, the features may be much less sensitive, the notion of label CDP is defined as follows.

\begin{defi}\label{def:labeldp}
	(Label CDP \cite{ghazi2021deep}) A randomized function $\mathcal{A}$ is $\epsilon$-label CDP if for any two datasets $D$ and $D'$ that differ on the label of only one training sample and any $S\subseteq \Theta$, \eqref{eq:dp} holds.
\end{defi}

Compared with Definition \ref{def:dp}, Definition \ref{def:labeldp} only requires the output to be insensitive to the replacement of a label. Therefore label CDP is a weaker requirement. Correspondingly, the label LDP is defined as follows.

\begin{defi}\label{def:local}
	(Label LDP \cite{ghazi2021deep}) A randomized function $M:(\mathcal{X}, \mathcal{Y})\rightarrow \mathcal{Z}$ is $\epsilon$-label LDP if
	\begin{eqnarray}
		\sup_{y,y'\in \mathcal{Y}}\sup_{S\subseteq \mathcal{Z}}\ln \frac{\text{P}(M(\mathbf{x}, y)\in S)}{\text{P}(M(\mathbf{x}, y')\in S)}\leq \epsilon.
		\label{eq:local}
	\end{eqnarray}
\end{defi}

Definition \ref{def:local} requires that each label is privatized locally before running any machine learning algorithms. It is straightforward to show that label LDP ensures label CDP. To be more precise, we have the following proposition.
\begin{prop}\label{prop:convert}
	Let $\mathbf{z}_i=M(\mathbf{x}_i,y_i)$ for $i=1,\ldots,N$. If $\mathcal{A}$ is a function of $(\mathbf{x}_i, \mathbf{z}_i)$, $i=1,\ldots, N$, then $\mathcal{A}$ is $\epsilon$-label CDP.	
\end{prop}

\subsection{Risk of Classification and Regression}
In supervised learning problems, given $N$ samples $(\mathbf{X}_i, Y_i)$, $i=1,\ldots, N$ drawn from a common distribution, the task is to learn a function $g:\mathcal{X}\rightarrow \mathcal{Y}$, in which $\mathcal{X}\subseteq \mathbb{R}^d$ is the feature domain, while $\mathcal{Y}\subseteq \mathbb{R}$ is the value domain. For a loss function $l:\mathcal{Y}\times \mathcal{Y}\rightarrow \mathbb{R}$, the goal is to minimize the \emph{risk function}, which is defined as the expectation of loss function between the predicted value and the ground truth:
\begin{eqnarray}
	R=\mathbb{E}[l(\hat{Y}, Y)].
	\label{eq:R}
\end{eqnarray}
The minimum risk among all function $g$ is called Bayes risk, i.e.
$R^*=\min_{g}\mathbb{E}[l(g(\mathbf{X}), Y)]$.
In practice, the sample distribution is unknown, and we need to learn $g$ from samples. Therefore, the risk of any practical classifiers is larger than the Bayes risk. The gap $R-R^*$ is called the excess risk, which is the target variable that we hope to minimize. Now we discuss classification and regression problems separately.

\emph{1) Classification.} For classification problems, the size of $\mathcal{Y}$ is finite. For convenience, we denote $\mathcal{Y}=[K]$, in which $[K]:=\{1,\ldots, K\}$. In this paper, we use $0-1$ loss, i.e. $l(\hat{Y},Y)=\mathbf{1}(\hat{Y}\neq Y)$, then $R=\text{P}(\hat{Y}\neq Y)$. Define $K$ functions $\eta_1,\ldots, \eta_K$ as the conditional class probabilities:
\begin{eqnarray}
	\eta_k(\mathbf{x})=\text{P}(Y=k|\mathbf{X}=\mathbf{x}), k=1,\ldots, K.
\end{eqnarray}
Under this setting, the Bayes optimal classifier and the corresponding Bayes risk is
\begin{eqnarray}
	c^*(\mathbf{x}) &=& \underset{j\in [K]}{\arg\max}\eta_j(\mathbf{x}),\label{eq:cstar}\\
	R_{cls}^*&=&\text{P}(c^*(\mathbf{X})\neq Y).\label{eq:Rstar}
\end{eqnarray}
\emph{2) Regression.} Now we consider the case with $\mathcal{Y}$ having infinite size. We use $\ell_2$ loss in this paper, i.e. $l(\hat{Y}, Y)=(\hat{Y}-Y)^2$. Then the Bayes risk is
\begin{eqnarray}
	R_{reg}^*=\mathbb{E}[(Y-\eta(\mathbf{X}))^2].
\end{eqnarray}

The following proposition provides a bound of the excess risk for classification and regression problems.
\begin{prop}\label{prop:excess}
	For any classifier $c:\mathcal{X}\rightarrow [K]$, the excess risk of classification is bounded by
	\begin{eqnarray}
		R_{cls}-R_{cls}^*=\int (\eta^*(\mathbf{x}) - \mathbb{E}[\eta_{c(\mathbf{x})}(\mathbf{x})]) f(\mathbf{x})d\mathbf{x},
		\label{eq:excess}
	\end{eqnarray}	
	in which $\eta^*(\mathbf{x}):=\max_k \eta_k(\mathbf{x})$ is the maximum regression function over all classes, and $f(\mathbf{x})$ is the probability density function (pdf) of random feature vector $\mathbf{X}$.
	For any regression estimate $\hat{\eta}:\mathcal{X}\rightarrow \mathcal{Y}$, the excess risk of regression with $\hat{Y}=\hat{\eta}(\mathbf{X})$ is bounded by
	\begin{eqnarray}
		R_{reg}-R_{reg}^*=\mathbb{E}[(\hat{\eta}(\mathbf{X})-\eta(\mathbf{X}))^2].
	\end{eqnarray}
\end{prop}
The proof of Proposition \ref{prop:excess} is shown in Appendix \ref{sec:excess}.

Finally, we state some basic assumptions that will be used throughout this paper.

\begin{ass}\label{ass:multiclass}
	There exists some constants $L$, $\beta$, $C_T$, $\gamma$, $c$, $D$ and $\theta\in (0,1]$ such that
	
	(a) For all $j\in [K]$ and any $\mathbf{x}$, $\mathbf{x}'$,
	$|\eta_j(\mathbf{x}) - \eta_j(\mathbf{x}')|\leq L\norm{\mathbf{x}-\mathbf{x}'}^\beta$;
	
	(b) (For classification only) For any $t>0$,
	$\text{P}\left(0<\eta^*(\mathbf{X}) - \eta_s(\mathbf{X})<t\right)\leq C_T t^\gamma,$
	in which $\eta_s(\mathbf{x})$ is the second largest one among $\{\eta_1(\mathbf{x}), \ldots, \eta_K(\mathbf{x})\}$;
	
	(c) The feature vector $\mathbf{X}$ has a pdf $f$ which is bounded from below, i.e. $f(\mathbf{x})\geq c$;
	
	(d) For all $r<D$, 
	$V_r(\mathbf{x})\geq \theta v_dr^d,$
	in which $V_r(\mathbf{x})$ is the volume (Lebesgue measure) of $B(\mathbf{x}, r)\cap \mathcal{X}$, $v_d$ is the volume of a unit ball.
\end{ass}

Assumption \ref{ass:multiclass} (a) requires that all $\eta_j$ are H{\"o}lder continuous. This condition is common in literatures about nonparametric statistics \citep{tsybakov2009introduction}. (b) is generalized from the Tsybakov noise assumption for binary classification, which is commonly used in many existing works in the field of both nonparametric classification \citep{audibert2007fast,chaudhuri2014rates,gadat2016classification,zhao2021minimax} and differential privacy \citep{berrett2019classification,berrett2021strongly}. If $K=2$, then $\eta^*$ and $\eta_s$ refer to the larger and smaller class conditional probability, respectively. An intuitive understanding of (b) is that in the majority of the support, the maximum value among $\{\eta_1(\mathbf{x}), \ldots, \eta_K(\mathbf{x}) \}$ should have some gap to the second largest one. With sufficiently large sample size and model complexity, assumption (b) ensures that for test samples within the majority of the support $\mathcal{X}$, the algorithm is highly likely to correctly identify the class with the maximum conditional probability. Therefore, in (b), we only care about $\eta^*(\mathbf{x})$ and $\eta_s(\mathbf{x})$, while other classes with small conditional probabilities can be ignored. (c) is usually called "strong density assumption" in existing works \citep{gadat2016classification,doring2018rate,zhao2024minimax}, which is quite strong. It is possible to relax this assumption so that the theoretical analysis becomes suitable for general cases. However, we do not focus on such generalization in this paper. Assumption (d) prevents the corner of the support $\mathcal{X}$ from being too sharp. In the remainder of this paper, we denote $\mathcal{F}_{cls}$ as the set of all pairs $(f, \eta)$ satisfying Assumption \ref{ass:multiclass}.

\section{Classification}\label{sec:cls}
In this section, we derive the upper and lower bounds of learning under central and label LDP, respectively.

\subsection{Local Label DP}

\emph{1) Lower bound.} The following theorem shows the minimax lower bound, which characterizes the theoretical limit. 

\begin{thm}\label{thm:clslb}
	Denote $\mathcal{M}_\epsilon$ as the set of all privacy mechanisms satisfying $\epsilon$-label LDP (Definition \ref{def:local}), then
	\begin{eqnarray}
		\underset{\hat{\eta}}{\inf}\underset{M\in \mathcal{M}_\epsilon}{\inf}\underset{(f,\eta)\in \mathcal{F}_{cls}}{\sup} (R_{cls}-R_{cls}^*) \gtrsim \left[N\left(\epsilon^2\wedge 1\right)\right]^{-\frac{\beta(\gamma+1)}{2\beta+d}},
		\label{eq:mmx}
	\end{eqnarray}
	in which $\inf_{\hat{\eta}}$ means to take infimum over all possible classifiers.
\end{thm}
\begin{proof}
	(Outline) It suffices to derive \eqref{eq:mmx} with $K=2$. We convert the problem into multiple binary hypothesis testing problems. In particular, we divide the support into $G$ cubes. For some of them, we construct two opposite hypotheses such that they are statistically not distinguishable. Our proof uses some techniques in local DP \citep{duchi2018minimax} and some classical minimax theory \citep{tsybakov2009introduction}. The detailed proof is shown in Appendix \ref{sec:clslb}.	
\end{proof}

In Theorem \ref{thm:clslb}, \eqref{eq:mmx} takes supremum over all joint distributions of $(\mathbf{X}, Y)$, and infimum over all classifiers and privacy mechanisms satisfying $\epsilon$-label LDP. 

\emph{2) Upper bound.} We now show that the bound \eqref{eq:mmx} is achievable. Let the privacy mechanism $M(\mathbf{x}, y)$ outputs a $K$ dimensional vector, with each component being either $0$ or $1$, such that
\begin{eqnarray}
	\text{P}(M(\mathbf{x}, y)(j)=1) = \left\{
	\begin{array}{ccc}
		\frac{e^\frac{\epsilon}{2}}{e^\frac{\epsilon}{2}+1} &\text{if} & y=j\\
		\frac{1}{e^\frac{\epsilon}{2}+1} &\text{if} & y\neq j,
	\end{array}
	\right.
	\label{eq:M}
\end{eqnarray} 
and $\text{P}(M(\mathbf{x}, y)(j)=0)=1-\text{P}(M(\mathbf{x}, y)(j)=1)$, in which $M(\mathbf{x}, y)(j)$ is the $j$-th component of $M(\mathbf{x}, y)$. For $N$ random training samples $(\mathbf{X}_i, Y_i)$, let $\mathbf{Z}_i=M(\mathbf{X}_i, Y_i)$, and correspondingly, $Z_i(j)$ is the $j$-th component of $\mathbf{Z}_i$.

Divide the support $\mathcal{X}$ into $G$ cubes, named $B_1,\ldots, B_G$, such that the side length of each cube is $h$. $B_1,\ldots, B_G$ are disjoint, and these cubes form a covering of $\mathcal{X}$, i.e.
$\mathcal{X}\subset \cup_{l=1}^G B_l$.
Then calculate
\begin{eqnarray}
	S_{lj} = \sum_{i:\mathbf{X}_i\in B_l} Z_i(j), l=1,\ldots, G, j=1,\ldots, K,
	\label{eq:Slj}
\end{eqnarray}

The classification within the $l$-th cube is
\begin{eqnarray}
	c_l=\underset{j}{\arg\max}S_{lj},
	\label{eq:cl}
\end{eqnarray}
such that the the prediction given $\mathbf{x}$ is $c(\mathbf{x})=c_l$ for all $\mathbf{x}\in B_l$. The next theorem shows the privacy guarantee, as well as the bound of the excess risk.
\begin{thm}\label{thm:clsub}
	The privacy mechanism $M$ is $\epsilon$-label LDP. Moreover, under Assumption \ref{ass:multiclass}, with
	$h\sim \left(N(\epsilon^2\wedge 1)/\ln  K\right)^{-\frac{1}{2\beta+d}}$,
	the excess risk of the classifier described above can be upper bounded as follows:
	\begin{eqnarray}
		R_{cls}-R_{cls}^*\lesssim \left(\frac{N(\epsilon^2 \wedge 1)}{\ln K}\right)^{-\frac{\beta(\gamma+1)}{2\beta+d}}.
		\label{eq:upper}
	\end{eqnarray}
\end{thm}

\begin{proof}
	(Outline) For privacy guarantee, we need to show that \eqref{eq:M} is $\epsilon$-label LDP:
	\begin{eqnarray}
		\frac{\text{P}(M(\mathbf{x}, y)=\mathbf{z})}{\text{P}(M(\mathbf{x}, y')=\mathbf{z})}&=&\Pi_{j=1}^K \frac{\text{P}(M(\mathbf{x}, y)(j)=\mathbf{z}(j))}{\text{P}(M(\mathbf{x}, y')(j)=\mathbf{z}(j))}\nonumber\\
		&=&\frac{\text{P}(M(\mathbf{x}, y)(y)=\mathbf{z}(y))}{\text{P}(M(\mathbf{x}, y')(y)=\mathbf{z}(y))}\frac{\text{P}(M(\mathbf{x}, y)(y')=\mathbf{z}(y'))}{\text{P}(M(\mathbf{x}, y')(y')=\mathbf{z}(y'))}\nonumber\\
		&\leq & e^\frac{\epsilon}{2} e^\frac{\epsilon}{2}=e^\epsilon.
	\end{eqnarray}
	According to Definition \ref{def:local}, $M$ is $\epsilon$-label LDP. For the performance guarantee \eqref{eq:upper}, according to Proposition \ref{prop:excess}, we need to bound $\eta^*(\mathbf{x})-\mathbb{E}[\eta_{c(\mathbf{x})}(\mathbf{x})]$ for each $\mathbf{x}$. If $\eta^*(\mathbf{x})-\eta_s(\mathbf{x})$ is large, then with high probability, $c(\mathbf{x})=c^*(\mathbf{x})$, and then $\eta^*(\mathbf{x}) = \eta_{c(\mathbf{x})}(\mathbf{x})$. Thus we mainly consider the case with small $\eta^*(\mathbf{x}) - \eta_s(\mathbf{x})$. The details of the proof are shown in Appendix \ref{sec:clsub}.
\end{proof}

The lower bound \eqref{eq:mmx} and the upper bound \eqref{eq:upper} match up to a logarithm factor, indicating that the results are tight. Now we comment on the results.
\begin{rmk}
	1) \textbf{Comparison with non-private bound.} The classical minimax lower bound for non-private classification problem is $N^{-\frac{\beta(\gamma+1)}{2\beta+d}}$. Therefore, the lower bound \eqref{eq:mmx} reaches the non-private bound with $\epsilon\gtrsim 1$. With small $\epsilon$, $N$ training samples with privatized labels roughly equals $N\epsilon^2$ non-privatized samples in terms of performance. 
	
	\noindent 2) \textbf{Comparison with LDP that protects both features and labels.} In this case, the optimal excess risk is $(N\epsilon^2)^{-\beta(\gamma+1)/(2\beta+2d)}\vee N^{-\beta(\gamma+1)/(2\beta+d)}$, which is worse than the right hand side of \eqref{eq:mmx}. Such a result indicates that compared with classical DP, label LDP incurs significantly weaker performance loss.
	
	\noindent 3) \textbf{Comparison with other baseline methods.} If we use the randomized response method instead of the privacy mechanism \eqref{eq:M}, then the performance decreases sharply with the number of classes $K$. Several methods have been proposed to improve the randomized response method, such as RRWithPrior \citep{ghazi2021deep} and ALIBI \citep{malek2021antipodes}. However, there is no theoretical guarantee for these methods.
\end{rmk} 

\subsection{Central Label DP}
\emph{1) Lower bound.} The following theorem shows the minimax lower bound under the label CDP.

\begin{thm}\label{thm:clslb-cdp}
	
	Denote $\mathcal{A}_\epsilon$ as the set of all learning algorithms satisfying $\epsilon$-label CDP (Definition \ref{def:labeldp}), then
	\begin{eqnarray}
		\underset{\mathcal{A}\in \mathcal{A}_\epsilon}{\inf}\underset{(f,\eta)\in \mathcal{F}_{cls}}{\sup} (R_{cls}-R_{cls}^*)\gtrsim N^{-\frac{\beta(\gamma+1)}{2\beta+d}}+(\epsilon N)^{-\frac{\beta(\gamma + 1)}{\beta + d}}.
		\label{eq:clslb-cdp}
	\end{eqnarray}
\end{thm}
\begin{proof}
	(Outline) Lower bounds under central DP are usually constructed by packing method \citep{hardt2010geometry}, which works for fixed output dimensions. However, to achieve a desirable bias and variance tradeoff, the model complexity needs to increase with $N$. In our proof, we still divide the support into $G$ cubes and construct two hypotheses for each cube, but we develop a new tool (see Lemma \ref{lem:p}) to give a lower bound of the minimum error of hypothesis testing. We then use the group privacy property \citep{dwork2014algorithmic} to obtain the overall lower bound. The details can be found in Appendix \ref{sec:clslb-cdp}.
\end{proof}

\emph{2) Upper bound.} Now we show that \eqref{eq:clslb-cdp} is achievable. Similar to the label LDP problem, now divide the support into $G$ cubes, such that the length of each cube is $h$. Now the classification within the $l$-th cube follows a exponential mechanism \citep{mcsherry2007mechanism}:
\begin{eqnarray}
	\text{P}(c_l=j|\mathbf{X}_{1:N}, Y_{1:N})=\frac{e^{\epsilon n_{lj}/2}}{\sum_{k=1}^K e^{\epsilon n_{lk}/2}},
	\label{eq:exp}
\end{eqnarray}
in which
$n_{lj}=\sum_{i=1}^N \mathbf{1}(\mathbf{X}_i\in B_l, Y_i=j)$.
Then let $c(\mathbf{x})=c_l$ for $\mathbf{x}\in B_l$. The excess risk is bounded in the next theorem.

\begin{thm}\label{thm:clsub-cdp}
	The privacy mechanism \eqref{eq:exp} is $\epsilon$-label CDP. Moreover, under Assumption \ref{ass:multiclass}, if $h$ scales as
	$h\sim \left(\ln K/\epsilon N\right)^\frac{1}{\beta+d}+(\ln K/N)^\frac{1}{2\beta + d}$,
	then the excess risk can be bounded as follows:
	\begin{eqnarray}
		R-R^*\lesssim \left(\frac{N}{\ln K}\right)^{-\frac{\beta(\gamma+1)}{2\beta+d}}+\left(\frac{\epsilon N}{\ln K}\right)^{-\frac{\beta(\gamma+1)}{\beta+d}}.
		\label{eq:excess2}
	\end{eqnarray}
\end{thm}
\begin{proof}
	(Outline) The privacy guarantee of the exponential mechanism has been analyzed in \citep{dwork2014algorithmic}. Following these existing analyses, it can be shown that \eqref{eq:exp} is $\epsilon$-label CDP. It remains to show \eqref{eq:excess2}. Note that if $\eta^*(\mathbf{x})-\eta_s(\mathbf{x})$ is large, then the difference between the largest and the second largest one from $\{n_{lj}|j=1,\ldots, K\}$ will also be large. From \eqref{eq:exp}, the following inequality holds with high probability:
	$c_l={\arg\max}_jn_{lj}={\arg\max}_j \eta_j(\mathbf{x})=c^*(\mathbf{x})$,
	which means that the classifier makes optimal prediction. Hence we mainly consider the case with small $\eta^*(\mathbf{x})-\eta_s(\mathbf{x})$. The details of the proof can be found in Appendix \ref{sec:clsub-cdp}.	
\end{proof}
The upper and lower bounds match up to logarithmic factors. In \eqref{eq:excess2}, the first term is just the non-private convergence rate, while the second term $(\epsilon N)^{-\frac{\beta(\gamma+1)}{\beta+d}}$ can be regarded as the additional risk caused by the privacy mechanism. It decays faster with $N$ compared with the first term, thus the additional performance loss caused by the privacy mechanism becomes negligible as $N$ increases. This result is crucially different from the local model, under which the privacy mechanism always induces higher sample complexity by a factor of $O(1/(\epsilon^2\wedge 1))$. 

\section{Regression with Bounded Noise}\label{sec:regb}
Now we analyze the theoretical limits of regression problems under local and label CDP. Throughout this section, we assume that the label is restricted within a bounded interval. 
\begin{ass}\label{ass:bounded}
	Given any $\mathbf{x}\in \mathcal{X}$, $\text{P}(|Y|<T|\mathbf{X}=\mathbf{x}) = 1$. 
\end{ass}

Assumption \ref{ass:multiclass} remains the same here. In the remainder of this section, denote $\mathcal{F}_{reg1}$ as the set of $(f, \eta)$ that satisfies Assumptions \ref{ass:multiclass} and \ref{ass:bounded}.

\subsection{Local Label DP}
\emph{1) Lower bound.} Theorem \ref{thm:reglb} shows the minimax lower bound.

\begin{thm}\label{thm:reglb}
	Denote $\mathcal{M}_\epsilon$ as the set of all privacy mechanisms satisfying $\epsilon$-label CDP, then
	\begin{eqnarray}
		\underset{\hat{\eta}}{\inf}\underset{M\in \mathcal{M}_\epsilon}{\inf}\underset{(f, \eta)\in \mathcal{F}_{reg1}}{\sup}(R_{reg}-R_{reg}^*)\gtrsim (N(\epsilon^2\wedge 1))^{-\frac{2\beta}{d+2\beta}}.
	\end{eqnarray}
\end{thm}

The proof of Theorem \ref{thm:reglb} is similar to that of Theorem \ref{thm:clslb}, except for some details in hypotheses construction and the final bound of excess risk. The details are shown in Appendix \ref{sec:reglb}.

\emph{2) Upper bound.} The privacy mechanism is
$Z=Y+W$,
in which $W\sim \Lap(2T/\epsilon)$. Then the privacy mechanism satisfies $\epsilon$-label LDP. In this case, the real regression function $\eta(\mathbf{x})$ can be estimated using the nearest neighbor approach. Let
\begin{eqnarray}
	\hat{\eta}(\mathbf{x}) = \frac{1}{k}\sum_{i\in \mathcal{N}_k(\mathbf{x})}Z_i,
	\label{eq:regest}
\end{eqnarray}
in which $\mathcal{N}_k(\mathbf{x})$ is the set of $k$ nearest neighbors of $\mathbf{x}$ among $\mathbf{X}_1,\ldots, \mathbf{X}_N$, in which distances between $\mathbf{x}$ and $\mathbf{X}_i$, $i=1,\ldots, N$ can be defined under arbitrary metric.

\begin{thm}\label{thm:regbounded}
	The method described above is $\epsilon$-label LDP. Moreover, with $k\sim N^{\frac{2\beta}{d+2\beta}} (\epsilon\wedge 1)^{-\frac{2d}{d+2\beta}}$, under Assumption \ref{ass:multiclass} and \ref{ass:bounded},
	\begin{eqnarray}
		R_{reg}-R_{reg}^*\lesssim (N(\epsilon^2\wedge 1))^{-\frac{2\beta}{d+2\beta}}.
		\label{eq:excess3}
	\end{eqnarray}
\end{thm}
\begin{proof}
	(Outline) Since $|Y|<T$, $W\sim \Lap(2T/\epsilon)$, it is obvious that $Z=Y+W$ is $\epsilon$-label LDP. For the performance \eqref{eq:excess3}, the bias can be bounded by the $k$ nearest neighbor distances based on Assumption \ref{ass:multiclass}(a). The variance of $\hat{\eta}(\mathbf{x})$ scales inversely with $k$. An appropriate $k$ can be selected to achieve a good tradeoff between bias and variance. The details are shown in Appendix \ref{sec:regbounded}.
\end{proof}
From standard minimax analysis on regression problems \cite{tsybakov2009introduction}, the non-private convergence rate is $N^{-2\beta/(d+2\beta)}$. From Theorems \ref{thm:reglb} and \ref{thm:regbounded}, the privatization process makes sample complexity larger by a $O(1/\epsilon^2)$ factor.

\subsection{Central Label DP}

\emph{1) Lower bound.} The following theorem shows the minimax lower bound.
\begin{thm}\label{thm:reglbcdp}
	Let $\mathcal{A}_\epsilon$ be the set of all algorithms satisfying $\epsilon$-central DP, then
	\begin{eqnarray}
		\underset{\mathcal{A}\in \mathcal{A}_\epsilon}{\inf}\underset{(f, \eta)\in \mathcal{F}_{reg1}}{\sup} (R_{reg}-R_{reg}^*)\gtrsim N^{-\frac{2\beta}{2\beta+d}}+(\epsilon N)^{-\frac{2\beta}{d+\beta}}.
	\end{eqnarray}
\end{thm}

\emph{2) Upper bound.} For each cube $B_l$, let $n_l = \sum_{i=1}^N \mathbf{1}(\mathbf{X}_i\in B_l)$ be the number of samples in $B_l$. If $n_l>0$, then
\begin{eqnarray}
	\hat{\eta}_l=\frac{1}{n_l}\sum_{i=1}^N \mathbf{1}(\mathbf{X}_i\in B_l)Y_i+W_l,
	\label{eq:etal}
\end{eqnarray}
in which $W_l\sim \Lap(2/(n_l\epsilon))$. If $n_l=0$, i.e. no sample falls in $B_l$, then just let $\hat{\eta}_l=0$. For all $\mathbf{x}\in B_l$, let $\hat{\eta}(\mathbf{x}) = \hat{\eta}_l$. The excess risk can be bounded with the following theorem.

\begin{thm}\label{thm:regubcdp}
	\eqref{eq:etal} is $\epsilon$-label CDP. Moreover, under Assumptions \ref{ass:multiclass} and \ref{ass:bounded}, if $h$ scales as
	$h \sim N^{-\frac{1}{2\beta + d}}+(\epsilon N)^{-\frac{1}{d+\beta}}$,
	then the excess risk is bounded by
	\begin{eqnarray}
		R-R^*\lesssim N^{-\frac{2\beta}{2\beta+d}}+(\epsilon N)^{-\frac{2\beta}{d+\beta}}.
		\label{eq:regubcdp}
	\end{eqnarray}
\end{thm}

The upper and lower bounds match, indicating that the results are tight. Again, the second term in \eqref{eq:regubcdp} converges faster than the first one with respect to $N$,  the performance loss caused by privacy constraints becomes negligible as $N$ increases. 

\section{Regression with Heavy-tailed Noise}\label{sec:regtail}
In this section, we consider the case such that the noise has tails. We make the following assumption.

\begin{ass}\label{ass:tail}
	For all $\mathbf{x}\in \mathcal{X}$, $\mathbb{E}[|Y|^p|\mathbf{X}=\mathbf{x}]\leq M_p$ for some $p\geq 2$.
\end{ass}

Instead of requiring $|Y|<T$ for some $T$, now we only assume that the $p$-th order moment is bounded. For non-private cases, given fixed noise variance, the tail does not affect the mean squared error of regression. As a result, as long as $p\geq 2$, the convergence rate of the regression risk is the same as the case with bounded noise. However, the label CDP requires the output to be insensitive to the worst case replacement of labels, which can be harder if the noise has tails. To achieve $\epsilon$-DP, the clipping radius decreases with $\epsilon$, thus the noise strength needs to grow faster than $O(1/\epsilon)$. As a result, the convergence rate becomes slower than the non-private case. In the remainder of this section, denote $\mathcal{F}_{reg2}$ as the set of $(f, \eta)$ that satisfies Assumptions \ref{ass:multiclass} and \ref{ass:tail}.

\subsection{Local Label DP}

\emph{1) Lower bound.} In earlier sections about classification and regression with bounded noise, the impact of privacy mechanisms is only a polynomial factor on $\epsilon$, while the convergence rate of excess risk with respect to $N$ is not changed. However, this rule no longer holds when the noise has heavy tails.
\begin{thm}\label{thm:taillb}
	Denote $\mathcal{M}_\epsilon$ as the set of all privacy mechanisms satisfying $\epsilon$-label CDP, then for small $\epsilon$,
	\begin{eqnarray}
		\underset{\hat{\eta}}{\inf}\underset{M\in \mathcal{M}_\epsilon}{\inf}\underset{(f, \eta)\in \mathcal{F}}{\sup}(R_{reg}-R_{reg}^*)\gtrsim (N(e^\epsilon - 1)^2)^{-\frac{2\beta(p-1)}{2p\beta+d(p-1)}}+N^{-\frac{2\beta}{2\beta+d}}.
	\end{eqnarray}
\end{thm}

\emph{2) Upper bound.} Since now the noise has unbounded distribution, without preprocessing, the sensitivity is unbounded, thus simply adding noise to $Y$ can no longer protect the privacy. Therefore, a solution is to clip $Y$ into $[-T, T]$, and add noise proportional to $T/\epsilon$ to achieve $\epsilon$-label LDP. Such truncation will inevitably introduce some bias. To achieve a tradeoff between clipping bias and sensitivity, the value of $T$ needs to be tuned carefully. Based on such intuition, the method is precisely stated as follows. Let
$Z_i=Y_i^T+W_i$,
in which $Y_i^T$ is the truncation of $Y_i$, i.e.
$Y_i^T=(Y_i\wedge T)\vee (-T)$,
and $W\sim \Lap(2T/\epsilon)$. The result is shown in the following theorem. 

\begin{thm}\label{thm:tail}
	The method above is $\epsilon$-label LDP. Moreover, with
	$k\sim (N\epsilon^2)^\frac{2p\beta}{2p\beta+d(p-1)}\vee N^\frac{2\beta}{2\beta+d}$,
	and
	$T\sim (k\epsilon^2)^\frac{1}{2p}$,
	the risk is bounded by
	\begin{eqnarray}
		R_{reg}-R_{reg}^* \lesssim (N\epsilon^2)^{-\frac{2\beta(p-1)}{2p\beta+d(p-1)}}+N^{-\frac{2\beta}{2\beta+d}}.
	\end{eqnarray}
\end{thm}
\begin{proof}
	(Outline) It can be shown that the clipping bias scales as $T^{2(1-p)}$. To meet the $\epsilon$-label CDP, an additional error that scales as $T/\epsilon$ is needed. By averaging over $k$ nearest neighbors, the variance caused by noise $W$ scales with $T^2/(k\epsilon^2)$. From standard analysis on nearest neighbor methods \citep{audibert2007fast}, the non-private mean squared error scales as $1/k+(k/N)^{2\beta/d}$. Put all these terms together, Theorem \ref{thm:tail} can be proved. Details can be found in Appendix \ref{sec:tail}.
\end{proof}

With the limit of $p\rightarrow\infty$, the problem reduces to the case with bounded noise, and the growth rate of $k$ and the convergence rate of risk are the same as those in Theorem \ref{thm:regbounded}. For finite $p$, $2\beta(p-1)/(2p\beta+d(p-1))<2\beta/(2\beta+d)$, thus the convergence rate becomes slower due to the privacy mechanism.
\subsection{Central Label DP}
\emph{1) Lower bound.} The minimax lower bound is shown in Theorem \ref{thm:taillbcdp}.
\begin{thm}\label{thm:taillbcdp}
	The minimax lower bound is
	\begin{eqnarray}
		\underset{\mathcal{A}\in \mathcal{A}_\epsilon}{\inf}\underset{(f,\eta)\in \mathcal{F}_{reg2}}{\sup} (R_{reg}-R_{reg}^*)\gtrsim N^{-\frac{2\beta}{2\beta+d}}+(\epsilon N)^{-\frac{2\beta (p-1)}{p\beta + d(p-1)}}.
	\end{eqnarray}
\end{thm}

\emph{2) Upper bound.} Now we derive the upper bound. To restrict the sensitivity, instead of estimating with \eqref{eq:etal} directly, now we calculate an average of clipped label values:
\begin{eqnarray}
	\hat{\eta}_l=\frac{1}{n_l}\sum_{i=1}^N \mathbf{1}(\mathbf{X}_i\in B_l)\Clip(Y_i, T)+W_l,
	\label{eq:etaltail}
\end{eqnarray}
in which $W_l\sim \Lap(2T/(n_l\epsilon))$. Then for all $\mathbf{x}\in B_l$, let $\hat{\eta}(\mathbf{x}) = \hat{\eta}_l$. The following theorem bounds the excess risk.
\begin{thm}\label{thm:tailubcdp}
	\eqref{eq:etaltail} is $\epsilon$-label CDP. Moreover, under Assumptions \ref{ass:multiclass} and \ref{ass:tail}, if $h$ and $T$ scale as
	$h \sim N^{-\frac{1}{2\beta+d}}+(\epsilon N)^{-\frac{1}{p\beta+d(p-1)}}$ and
	$T\sim (\epsilon Nh^d)^{1/p}$,
	then the excess risk can be bounded by
	\begin{eqnarray}
		R_{reg}-R_{reg}^*\lesssim N^{-\frac{2\beta}{2\beta+d}} + (\epsilon N)^{-\frac{2\beta(p-1)}{p\beta+d(p-1)}}.
		\label{eq:tailubcdp}
	\end{eqnarray}
\end{thm}
The proof of Theorems \ref{thm:taillbcdp} and \ref{thm:tailubcdp} follow that of Theorems \ref{thm:reglbcdp} and \ref{thm:regubcdp}. The details are shown in Appendix \ref{sec:taillbcdp} and \ref{sec:tailubcdp} respectively. With $p=2$, the right hand side of \eqref{eq:tailubcdp} becomes $(\epsilon\wedge 1)^{-\frac{2\beta}{2\beta+d}}$, indicating that the privacy constraint blows up the sample complexity by a constant factor. With a larger $p$, the second term in \eqref{eq:tailubcdp} becomes negligible compared with the first one.

The theoretical analyses in this section are summarized as follows. In general, with fixed noise variance, if the label noise is heavy-tailed, while the non-private convergence rates remain unaffected, the additional risk caused by privacy mechanisms becomes significantly higher, indicating the difficulty of privacy protection for heavy-tailed distributions.

\section{Comparison with Central DP on Both Features and Labels}\label{sec:full}
This section analyzes the bound for central full DP, i.e. protecting both features and labels. For both classification and regression problems, we only need to show upper bounds. For lower bounds, the results of the label CDP can be directly used here. Denote $\mathcal{A}_\epsilon^0$ as the set of all learning algorithms satisfying $\epsilon$-DP (Definition \ref{def:dp}), while $\mathcal{A}_\epsilon$ is defined as the set of all algorithms satisfying $\epsilon$-Label DP (Definition \ref{def:labeldp}). Any algorithms satisfying $\epsilon$-DP must also satisfy $\epsilon$-label CDP, thus $\mathcal{A}_\epsilon^0\subseteq \mathcal{A}_\epsilon$. Therefore
\begin{eqnarray}
	\underset{\mathcal{A}\in \mathcal{A}_\epsilon^0}{\inf}\underset{(f,\eta)\in \mathcal{F}}{\sup} (R-R^*)\geq 	\underset{\mathcal{A}\in \mathcal{A}_\epsilon}{\inf}\underset{(f,\eta)\in \mathcal{F}}{\sup} (R-R^*),
\end{eqnarray}
in which the definition of $R$ and $\mathcal{F}$ depends on the tasks. Therefore, the lower bounds under $\epsilon$-label CDP also imply the lower bounds under $\epsilon$-DP.

\subsection{Classification} 
We use a modified version of \eqref{eq:exp}:
\begin{eqnarray}
	\text{P}(c_l=j|\mathbf{X}_{1:N}, Y_{1:N})=\frac{e^{\epsilon n_{lj}/4}}{\sum_{k=1}^K e^{\epsilon n_{lk}/4}}.
	\label{eq:fullcls}
\end{eqnarray}
Compared with \eqref{eq:exp}, the only difference is that now we have replaced $\epsilon/2$ with $\epsilon/4$. The analysis of \eqref{eq:fullcls} is shown in Theorem \ref{thm:fullcls}.
\begin{thm}\label{thm:fullcls}
	The estimation \eqref{eq:fullcls} is $\epsilon$-DP. Moreover, under Assumption \ref{ass:multiclass}, if $h$ scales as $h\sim (\ln K/(\epsilon N))^\frac{1}{\beta+d}+(\ln K/N)^\frac{1}{2\beta+d}$, the excess risk is bounded by
	\begin{eqnarray}
		R-R^*\lesssim \left(\frac{N}{\ln K}\right)^{-\frac{\beta(\gamma+1)}{2\beta+d}}+\left(\frac{\epsilon N}{\ln K}\right)^{-\frac{\beta(\gamma+1)}{\beta+d}}.
		\label{eq:riskfullcls}
	\end{eqnarray}
\end{thm}
\begin{proof} (Outline) We mainly prove that \eqref{eq:fullcls} is $\epsilon$-DP. For the performance guarantee, the proof of \eqref{eq:riskfullcls} just follows that of Theorem \ref{thm:clsub-cdp}. The detailed proof is shown in Appendix \ref{sec:fullcls}.
\end{proof}
The risk bound is the same as Theorem \ref{thm:clsub-cdp}. The only difference between \eqref{eq:fullcls} and \eqref{eq:exp} is that now the label noise is enlarged as if the privacy budget is reduced by half, thus the risk only increases up to a constant factor.

\subsection{Regression}
For regression problems, we need to control the overall sensitivity. \eqref{eq:etal} and \eqref{eq:etaltail} can be very large for small $n_l$. To cope with this issue, we propose a new estimator as follows. For bounded noise (i.e. under Assumption \ref{ass:bounded}, $|Y|<T$ holds), we let
\begin{eqnarray}
	\hat{\eta}_l = \frac{1}{n_l\vee n_0}\sum_{i=1}^N \mathbf{1}(\mathbf{X}_i\in B_l) Y_i+W_l.
	\label{eq:fullreg}
\end{eqnarray}
For unbounded noise (i.e. under Assumption \ref{ass:tail}), we set
\begin{eqnarray}
	\hat{\eta}_l = \frac{1}{n_l\vee n_0}\sum_{i=1}^N \mathbf{1}(\mathbf{X}_i\in B_l)\Clip(Y_i, T)+W_l,
	\label{eq:fullreg2}
\end{eqnarray}
in which
\begin{eqnarray}
	n_0=\frac{1}{2}Nc\theta h^d,
\end{eqnarray}
and
\begin{eqnarray}
	W_l \sim \Lap\left(\frac{6T}{n_0\epsilon}\right).
\end{eqnarray}
\begin{thm}\label{thm:fullreg}
	\eqref{eq:fullreg} and \eqref{eq:fullreg2} are $\epsilon$-DP, and
	
	(1) For the case with bounded noise, using \eqref{eq:fullreg}, under Assumptions \ref{ass:multiclass} and \ref{ass:bounded}, if $h$ scales as $h\sim N^{-\frac{1}{2\beta + d}} + (\epsilon N)^{-\frac{1}{d+\beta}}$, then
	\begin{eqnarray}
		R_{reg}-R_{reg}^*\lesssim N^{-\frac{2\beta}{2\beta+d}} + (\epsilon N)^{-\frac{2\beta}{d+\beta}};
	\end{eqnarray}
	
	(2) For the case with unbounded noise, using \eqref{eq:fullreg2}, if $h$ and $T$ scale as $h \sim N^{-\frac{1}{2\beta+d}}+(\epsilon N)^{-\frac{1}{p\beta+d(p-1)}}$
	and
	$T\sim (\epsilon Nh^d)^{1/p}$, then the excess risk can be bounded by
	\begin{eqnarray}
		R_{reg}-R_{reg}^*\lesssim N^{-\frac{2\beta}{2\beta+d}} + (\epsilon N)^{-\frac{2\beta(p-1)}{p\beta+d(p-1)}}.
		\label{eq:tailfull}
	\end{eqnarray}
\end{thm}
Similar to the classification case, the bounds in Theorem \ref{thm:fullreg} match \eqref{eq:regubcdp} and \eqref{eq:tailubcdp}. These results indicate that compared with label CDP, the risk of full CDP on both feature and label is larger only up to a constant factor.

\section{Conclusion}\label{sec:conc}

In this paper, we have analyzed the theoretical limits of learning under label DP. We have derived minimax rates under both central and local models, including lower bounds and matching upper bounds. Under the local model, the risk under label LDP converges faster than that under full LDP, indicating that the relaxation of LDP definition to only label yields significant performance improvement. In contrast, under the central model, the risk of label CDP is only reduced by a constant factor compared with full CDP, indicating that the relaxation of CDP only has limited benefits on the performance.

\nocite{zhang2025tutorial}

\bibliographystyle{ieeetr}
\bibliography{labeldp}


\appendices

\section{Proof of Proposition \ref{prop:excess}}\label{sec:excess}
\noindent \textbf{Classification.} From \eqref{eq:cstar} and \eqref{eq:Rstar}, the Bayes risk is
\begin{eqnarray}
	R_{cls}^*=\text{P}(Y\neq c^*(\mathbf{X}))=\int \text{P}(Y\neq c^*(\mathbf{x})|\mathbf{X}=\mathbf{x}) f(\mathbf{x}) d\mathbf{x}=  \int (1-\eta^*(\mathbf{x})) f(\mathbf{x})d\mathbf{x}.
\end{eqnarray}
The risk of classifier $c$ is
\begin{eqnarray}
	R_{cls}=\text{P}(Y\neq c(\mathbf{X}))=\mathbb{E}\left[\int \left(1-\eta_{c(\mathbf{x})}(\mathbf{x})\right)f(\mathbf{x})d\mathbf{x}\right].
	\label{eq:R2}
\end{eqnarray}
From \eqref{eq:R2} and \eqref{eq:Rstar},
\begin{eqnarray}
	R_{cls}-R_{cls}^*=\int (\eta^*(\mathbf{x}) - \mathbb{E}[\eta_{c(\mathbf{x})}(\mathbf{x})]) f(\mathbf{x})d\mathbf{x}.
\end{eqnarray}

\noindent \textbf{Regression.}
\begin{eqnarray}
	R_{reg} &=& \mathbb{E}[(\hat{Y}-Y)^2]\nonumber\\
	&=& \mathbb{E}\left[\left(\hat{\eta}(\mathbf{X}) - Y\right)^2\right]\nonumber\\
	&=&\mathbb{E}\left[(\hat{\eta}(\mathbf{X}) - \eta(\mathbf{X}) + \eta(\mathbf{X}) - Y)^2\right]\nonumber\\
	&=&\mathbb{E}\left[\left(\hat{\eta}(\mathbf{X}) - \eta(\mathbf{X})\right)^2\right] + \mathbb{E}\left[(\eta(\mathbf{X}) - Y)^2\right]\nonumber\\
	&=& R_{reg}^* + \mathbb{E}\left[\left(\hat{\eta}(\mathbf{X}) - \eta(\mathbf{X})\right)^2\right],
\end{eqnarray}
in which the expectation is taken over the randomness of $\mathbf{X}$ and $\mathbf{Y}$. $\hat{\eta}$ is assumed to be fixed here. The proof is complete.

\section{Proof of Theorem \ref{thm:clslb}}\label{sec:clslb}
In this section, we prove the minimax lower bound of multi-class classification. The problem with $K$ classes with $K>2$ is inherently harder than that with $K=2$. Therefore, we only need to prove the lower bound for binary classification, in which $\mathcal{Y}=\{1,2\}$. Let
\begin{eqnarray}
	\eta(\mathbf{x}) = \eta_2(\mathbf{x}) - \eta_1(\mathbf{x}).
\end{eqnarray}
Since $\eta_1(\mathbf{x})+\eta_2(\mathbf{x})=1$ always holds, we have 
\begin{eqnarray}
	\eta_1(\mathbf{x}) = \frac{1-\eta(\mathbf{x})}{2}, \eta_2(\mathbf{x}) = \frac{1+\eta(\mathbf{x})}{2}.
\end{eqnarray}
Therefore, $\eta(\mathbf{x})$ captures the conditional distribution of $Y$ given $\mathbf{x}$. 

Find $G$ disjoint cubes $B_1,\ldots, B_G\subset \mathcal{X}$, such that the length of each cube is $h$. Denote $\mathbf{c}_1,\ldots, \mathbf{c}_G$ as the centers of these cubes. Let $\phi(\mathbf{u})$ be some function supported at $[-1/2,1/2]^d$, such that
\begin{eqnarray}
	0\leq \phi(\mathbf{u})\leq 1.
	\label{eq:phirange}
\end{eqnarray}
Let $f(\mathbf{x}) = c$ over $\mathbf{x}\in \mathcal{X}$. For $\mathbf{v}\in \mathcal{V}:=\{-1,1\}^m$, let
\begin{eqnarray}
	\eta_\mathbf{v}(\mathbf{x})=\sum_{k=1}^m v_k\phi\left(\frac{\mathbf{x}-\mathbf{c}_k}{h}\right) h^\beta.
	\label{eq:etav}
\end{eqnarray}
It can be proved that if for some constant $C_M$,
\begin{eqnarray}
	m\leq C_Mh^{\gamma\beta-d},
	\label{eq:mbound}
\end{eqnarray}
then for any $\eta=\eta_\mathbf{v}$,  $\eta_1$ and $\eta_2$ satisfy Assumption \ref{ass:multiclass}(b). Denote
\begin{eqnarray}
	\hat{v}_k=\underset{s\in \{-1,1\}}{\arg\max}\;\int_{B_k} \phi\left(\frac{\mathbf{x}-\mathbf{c}_k}{h}\right) \mathbf{1}(\sign(\hat{\eta}(\mathbf{x}))=s) f(\mathbf{x})d\mathbf{x}.
	\label{eq:vkhat}
\end{eqnarray}
Then the excess risk is bounded by
\begin{eqnarray}
	R-R^*&=&\int |\eta_\mathbf{v}(\mathbf{x})|\text{P}(\sign(\hat{\eta}(\mathbf{x}))\neq \sign(\eta_\mathbf{v}(\mathbf{x}))) f(\mathbf{x})d\mathbf{x}\nonumber\\
	&\geq & \sum_{k=1}^m \int_{B_k} |\eta_\mathbf{v}(\mathbf{x})|\text{P}(\sign(\hat{\eta}(\mathbf{x}))\neq \sign(\eta_\mathbf{v}(\mathbf{x})))f(\mathbf{x})d\mathbf{x}\nonumber\\
	&=&\sum_{k=1}^m h^\beta \int_{B_k} \phi\left(\frac{\mathbf{x}-\mathbf{c}_k}{h}\right) \text{P}(\sign(\hat{\eta}(\mathbf{x})))f(\mathbf{x})d\mathbf{x}.
\end{eqnarray}
If $\hat{v}_k\neq v_k$, then from \eqref{eq:vkhat},
\begin{eqnarray}
	\int_{B_k} \phi\left(\frac{\mathbf{x}-\mathbf{c}_k}{h}\right)\mathbf{1}(\sign(\hat{\eta}(\mathbf{x})))f(\mathbf{x})d\mathbf{x}\geq \int_{B_k} \phi\left(\frac{\mathbf{x}-\mathbf{c}_k}{h}\right) \mathbf{1}(\sign(\hat{\eta}(\mathbf{x}))=v_k) f(\mathbf{x})d\mathbf{x}.
\end{eqnarray}
Therefore
\begin{eqnarray}
	\int_{B_k} \phi\left(\frac{\mathbf{x}-\mathbf{c}_k}{h}\right) \mathbf{1}(\sign(\hat{\eta}(\mathbf{x}))\neq v_k) f(\mathbf{x})d\mathbf{x}\geq \frac{1}{2}\int_{B_k} \phi\left(\frac{\mathbf{x}-\mathbf{c}_k}{h}\right) f(\mathbf{x})d\mathbf{x}\geq \frac{1}{2}ch^d \norm{\phi}_1.
\end{eqnarray}
Hence
\begin{eqnarray}
	R-R^*&\geq& \frac{1}{2}ch^{\beta+d}\norm{\phi}_1\sum_{k=1}^m \text{P}(\hat{v}_k\neq v_k)\nonumber\\
	&=&\frac{1}{2}ch^{\beta+d}\norm{\phi}_1\mathbb{E}[\rho_H(\hat{\mathbf{v}}, \mathbf{v})],
	\label{eq:excessexp}
\end{eqnarray}
in which $\rho_H$ denotes the Hamming distance. Then
\begin{eqnarray}
	\underset{\hat{Y}}{\inf}\underset{M\in \mathcal{M}_\epsilon}{\inf} \underset{(f,\eta)\in \mathcal{P}}{\sup} (R-R^*)\geq \frac{1}{2}h^{\beta+d}\norm{\phi}_1\underset{\hat{\mathbf{v}}}{\inf}\underset{M\in \mathcal{M}_\epsilon}{\inf}\underset{\mathbf{v}\in \mathcal{V}}{\max}\mathbb{E}[\rho_H(\hat{\mathbf{v}}, \mathbf{v})].
\end{eqnarray}

Define
\begin{eqnarray}
	\delta=\underset{M\in \mathcal{M}_\epsilon}{\sup}\underset{\mathbf{v}, \mathbf{v}':\rho_H(\mathbf{v}, \mathbf{v}')= 1}{\max}D_{KL}(P_{(X,Z)_{1:N}|\mathbf{v}}||P_{(X,Z)_{1:N}|\mathbf{v}'}),
\end{eqnarray}
in which $P_{(X,Z)_{1:N}|\mathbf{v}}$ denotes the distribution of $(\mathbf{X}_1,Z_1), \ldots, (\mathbf{X}_N, Z_N)$ with $\eta=\eta_\mathbf{v}$. $D_{KL}$ denotes the KL divergence. Then from \citep{tsybakov2009introduction}, Theorem 2.12(iv),
\begin{eqnarray}
	\underset{\hat{\mathbf{v}}}{\inf}\hspace{1mm} \underset{M}{\inf}\hspace{1mm} \underset{\mathbf{v}\in \mathcal{V}}{\max}\hspace{1mm}\mathbb{E}[\rho_H(\hat{\mathbf{v}}, \mathbf{v})]\geq \frac{m}{2}\left(\frac{1}{2}e^{-\delta}, 1-\sqrt{\frac{\delta}{2}}\right).
	\label{eq:lb-tsy}
\end{eqnarray}
It remains to bound $\delta$. Without loss of generality, suppose $v_1\neq v_1'$, and $v_i=v_i'$ for $i\neq 1$. Then
\begin{eqnarray}
	D_{KL}(P_{(X,Z)_{1:N}|\mathbf{v}}||P_{(X,Z)_{1:N}|\mathbf{v}'})&\overset{(a)}{=}& ND_{KL}(P_{X,Z|\mathbf{v}}||P_{X,Z|\mathbf{v}'})\nonumber\\
	&\overset{(b)}{=}&N\int_{B_1} f(\mathbf{x})D_{KL}(P_{Z|\mathbf{X}=\mathbf{x}, \mathbf{v}}||P_{Z|\mathbf{X}=\mathbf{x}, \mathbf{v}'})d\mathbf{x}\nonumber\\
	&\overset{(c)}{\leq} & N\int_{B_1} f(\mathbf{x})(e^\epsilon-1)^2 \mathbb{TV}^2 (P_{Z|\mathbf{X}=\mathbf{x}, \mathbf{v}}, P_{Z|\mathbf{X}=\mathbf{x}, \mathbf{v}'})d\mathbf{x}\nonumber\\
	&=& N\int_{B_1} f(\mathbf{x})(e^\epsilon-1)^2 \eta_\mathbf{v}^2(\mathbf{x})d\mathbf{x}\nonumber\\
	&=& N(e^\epsilon-1)^2 \int_{B_1} f(\mathbf{x})\phi^2 \left(\frac{\mathbf{x}-\mathbf{c}_1}{h}\right) h^{2\beta}d\mathbf{x}\nonumber\\
	&\overset{(d)}{=}& N(e^\epsilon-1)^2 h^{2\beta+d}\norm{\phi}_2^2.
	\label{eq:klbound1}
\end{eqnarray}
In (a), $P_{X,Z|\mathbf{v}}$ denotes the distribution of a single sample with privatized label $(X, Z)$, with $\eta=\eta_\mathbf{v}$. In (b), $P_{Z|\mathbf{X}=\mathbf{x}, \mathbf{v}}$ denotes the conditional distribution of $Z$ given $\mathbf{X}=\mathbf{x}$, with $\eta=\eta_\mathbf{v}$. (c) uses \citep{duchi2018minimax}, Theorem 1. In (d), $\norm{\phi}_2^2=\int \phi^2(\mathbf{u})d\mathbf{u}$, which is a constant. Moreover,
\begin{eqnarray}
	D_{KL}(P_{X,Z|\mathbf{v}}||P_{X,Z|\mathbf{v}'})&\overset{(a)}{\leq}& D_{KL}(P_{X,Y|\mathbf{v}}||P_{X,Y|\mathbf{v}'})\nonumber\\
	&=&\int_{B_1}f(\mathbf{x})\left[\text{P}(Y=1|\mathbf{v})\ln \frac{\text{P}(Y=1|\mathbf{v})}{\text{P}(Y=1|\mathbf{v}')}+\text{P}(Y=-1|\mathbf{v})\ln \frac{\text{P}(Y=-1|\mathbf{v})}{\text{P}(Y=-1|\mathbf{v}')}\right]d\mathbf{x}\nonumber\\
	&=& \int_{B_1} f(\mathbf{x})\left[\frac{1+\eta_\mathbf{v}(\mathbf{x})}{2}\ln\frac{1+\eta_\mathbf{v}(\mathbf{x})}{1-\eta_\mathbf{v}(\mathbf{x})}+\frac{1-\eta_\mathbf{v}(\mathbf{x})}{2}\ln \frac{1-\eta_\mathbf{v}(\mathbf{x})}{1+\eta_\mathbf{v}(\mathbf{x})}\right] d\mathbf{x}\nonumber\\
	&\overset{(b)}{\leq} & 3\int_{B_1} f(\mathbf{x})\eta_\mathbf{v}^2(\mathbf{x})d\mathbf{x}\nonumber\\
	&\leq & 3h^{2\beta+d}\norm{\phi}_2^2.
	\label{eq:klbound2}
\end{eqnarray}
For (a), note that $Z$ is generated from $Y$. From data processing inequality, (a) holds. For (b), without loss of generality, suppose that $v_1=1$, thus $\eta_\mathbf{v}(\mathbf{x})\geq 0$ in $B_1$. Then $\ln (1+\eta_\mathbf{v}(\mathbf{x}))\leq \eta_\mathbf{v}(\mathbf{x})$. From \eqref{eq:phirange} and \eqref{eq:etav}, $|\eta_\mathbf{v}(\mathbf{x})|\leq 1/2$. Therefore, $-\ln (1-\eta_\mathbf{v}(\mathbf{x}))\leq 2\eta_\mathbf{v}(\mathbf{x})$. Therefore (b) holds.

From \eqref{eq:klbound1} and \eqref{eq:klbound2},
\begin{eqnarray}
	\delta\leq N\left[(e^\epsilon-1)^2 \wedge 3\right] h^{2\beta+d}\norm{\phi}_2^2.
\end{eqnarray}
Let 
\begin{eqnarray}
	h\sim \left(N\left(\epsilon^2\wedge 1\right)\right)^{-\frac{1}{2\beta+d}}.
\end{eqnarray}
Then $\delta \lesssim 1$. From \eqref{eq:lb-tsy}, with $m\sim h^{\gamma \beta-d}$,
\begin{eqnarray}
	\underset{\hat{\mathbf{v}}}{\inf}\underset{M\in \mathcal{M}_\epsilon}{\inf}\underset{\mathbf{v}\in \mathcal{V}}{\max}\mathbb{E}[\rho_H(\hat{\mathbf{v}}, \mathbf{v})]\gtrsim h^{\gamma \beta-d}.
\end{eqnarray}
Hence
\begin{eqnarray}
	\underset{\hat{Y}}{\inf}\underset{M\in \mathcal{M}_\epsilon}{\inf} \underset{(f,\eta)\in \mathcal{P}}{\sup} (R-R^*)\gtrsim h^{\beta+d} h^{\gamma\beta-d}\sim h^{\beta(\gamma+1)}\sim \left[N\left(\epsilon^2\wedge 1\right)\right]^{-\frac{\beta(\gamma+1)}{2\beta+d}}.
\end{eqnarray}
The proof is complete.

\section{Proof of Theorem \ref{thm:clsub}}\label{sec:clsub}
Denote
\begin{eqnarray}
	n_l=\sum_{i=1}^N \mathbf{1}(\mathbf{X}_i\in B_l),
\end{eqnarray}
and for $\mathbf{Z}=M(\mathbf{X}, Y)$, let
\begin{eqnarray}
	\tilde{\eta}_j(\mathbf{x})&:=&\mathbb{E}[\mathbf{Z}(j)|\mathbf{X}=\mathbf{x}]\nonumber\\
	&=&\frac{e^\frac{\epsilon}{2}}{e^\frac{\epsilon}{2}+1} \eta_j(\mathbf{x})+ \frac{1}{e^\frac{\epsilon}{2}+1}(1-\eta_j(\mathbf{x}))
	\label{eq:convert}
\end{eqnarray}
as the number of training samples whose feature vectors fall in $B_l$, and
\begin{eqnarray}
	v_{lj}:=\frac{1}{n_l}\sum_{i:\mathbf{X}_i\in B_l} \tilde{\eta}_j(\mathbf{X}_i).
	\label{eq:vlj}
\end{eqnarray}
Recall \eqref{eq:Slj} that defines $S_{lj}$. From Hoeffding's inequality,
\begin{eqnarray}
	\text{P}\left(|S_{lj}-n_lv_{lj}|>t|\mathbf{X}_{1:N}\right)\leq 2\exp\left[-\frac{2t^2}{n_l}\right],
\end{eqnarray}
in which $\mathbf{X}_{1:N}$ denotes $\mathbf{X}_1,\ldots, \mathbf{X}_N$.

Define
\begin{eqnarray}
	v_l^*:=\underset{j}{\max}v_{lj},
\end{eqnarray}
and
\begin{eqnarray}
	c_l^*:=\underset{j}{\arg\max}\;v_{lj}.
\end{eqnarray}
Now we bound $\text{P}(v_l^*-v_{lc_l}>t)$, in which $c_l$ is defined in \eqref{eq:cl}. $c_l$ can be viewed as the prediction at the $l$-th cube. We would like to show that the even if the prediction is wrong, the value (i.e. conditional probability) of the predicted class is close to the ground truth. $v_l^*-v_{lc_l}>t$ only if $\exists j$, $v_l^*-v_{lj}>t$, and $S_{lj}>S_{lc_l^*}$. Therefore either $S_{lj}-n_lv_{lj}>t/2$ or $S_{lc_l^*}-n_lv_l^*>t/2$ holds. Hence

\begin{eqnarray}
	\text{P}\left(v_l^*-v_{lc_l}\geq t\right) \leq \text{P}\left(\exists j, |S_{lj} - n_lv_{lj}|\geq \frac{1}{2}n_l t\right)\leq 2K\exp\left(-\frac{1}{2}n_lt^2\right).
\end{eqnarray}
Define
\begin{eqnarray}
	t_0=\sqrt{\frac{2\ln(2K)}{n_l}}.
\end{eqnarray}
Then
\begin{eqnarray}
	v_l^*-\mathbb{E}[v_{lc_l}|\mathbf{X}_{1:N}] &=&\int_0^1 \text{P}(v_l^*-v_{lc_l}> t)dt\nonumber\\
	&\leq & t_0+\int_{t_0}^\infty 2K\exp\left(-\frac{1}{2}n_lt^2\right) dt\nonumber\\
	&\overset{(a)}{\leq} &	t_0+2\sqrt{\frac{2\pi}{n_l}}K\exp\left(-\frac{1}{2}n_lt_0^2\right)\nonumber\\
	&=&\sqrt{\frac{2\ln(2K)}{n_l}}+\sqrt{\frac{2\pi}{n_l}}\nonumber\\
	&\leq & 3\sqrt{\frac{\ln (2K)}{n_l}}.
\end{eqnarray}
In (a), we use the inequality
\begin{eqnarray}
	\int_t^\infty e^{-\frac{u^2}{2\sigma^2}} du \leq \sqrt{2\pi}\sigma e^{-\frac{t^2}{2\sigma^2}}.
\end{eqnarray}

Now we bound the excess risk.
\begin{eqnarray}
	R-R^*&=& \int \left(\eta^*(\mathbf{x}) - \mathbb{E}[\eta_{c(\mathbf{x})}(\mathbf{x})]\right) f(\mathbf{x})d\mathbf{x}\nonumber\\
	&=&\sum_{l=1}^G \int_{B_l}\left(\eta^*(\mathbf{x}) - \mathbb{E}[\eta_{c(\mathbf{x})}(\mathbf{x})]\right) f(\mathbf{x})d\mathbf{x}.
\end{eqnarray}
We need to bound $\int_{B_l}\left(\eta^*(\mathbf{x}) - \mathbb{E}[\eta_{c(\mathbf{x})}(\mathbf{x})]\right) f(\mathbf{x})d\mathbf{x}$ for each $l$. From Assumption \ref{ass:multiclass}(a), for any $\mathbf{x}, \mathbf{x}'\in B_l$, the distance is bounded by $\norm{\mathbf{x}-\mathbf{x}'}\leq \sqrt{d}L$. Thus
\begin{eqnarray}
	|\eta_j(\mathbf{x})-\eta_j(\mathbf{x}')|\leq L_dh^\beta,
	\label{eq:maxdiff}
\end{eqnarray}
in which $L_d$ is defined as $L_d:=L\sqrt{d}$. From \eqref{eq:maxdiff} and \eqref{eq:convert},
\begin{eqnarray}
	|\tilde{\eta}_j(\mathbf{x}) - \tilde{\eta}_j(\mathbf{x}')|\leq \frac{e^\frac{\epsilon}{2}-1}{e^\frac{\epsilon}{2}+1}L_d h^\beta.
\end{eqnarray}
Define
\begin{eqnarray}
	\tilde{\eta}^*(\mathbf{x}) = \underset{j}{\max}\tilde{\eta}_j(\mathbf{x}),
\end{eqnarray}
then
\begin{eqnarray}
	\eta^*(\mathbf{x}) - \mathbb{E}[\eta_{c_l}(\mathbf{x})|\mathbf{X}_{1:N}]&\leq &\frac{e^\frac{\epsilon}{2}+1}{e^\frac{\epsilon}{2}-1}\left(\tilde{\eta}^*(\mathbf{x}) - \mathbb{E}[\tilde{\eta}_{c_l}(\mathbf{x})|\mathbf{X}_{1:N}]\right)\nonumber\\
	&\leq & \frac{e^\frac{\epsilon}{2}+1}{e^\frac{\epsilon}{2}-1}\left(v_l^*-\mathbb{E}[v_{lc_l}|\mathbf{X}_{1:N}]\right)+2L_dh^\beta\nonumber\\
	&\leq & 3\frac{e^\frac{\epsilon}{2}+1}{e^\frac{\epsilon}{2}-1}\sqrt{\frac{2\ln (2K)}{n_l}}+2L_dh^\beta.
\end{eqnarray}
Take integration over cube $B_l$, we get
\begin{eqnarray}
	&&\int_{B_l} \left(\eta^*(\mathbf{x}) - \mathbb{E}[\eta_{c_l}(\mathbf{x})]\right) f(\mathbf{x})d\mathbf{x}\nonumber\\
	&\leq& \text{P}\left(n_l<\frac{1}{2}Np(B_l)\right)\int_{B_l} \left(\eta^*(\mathbf{x}) - \mathbb{E}[\eta_{c_l}(\mathbf{x})|n_l<\frac{1}{N}p(B_l)]\right) f(\mathbf{x})d\mathbf{x}\nonumber\\
	&&+\int_{B_l}\left(\eta^*(\mathbf{x}) -\mathbb{E}[\eta_{c_l}(\mathbf{x})|n_l\geq \frac{1}{N}p(B_l)]\right) f(\mathbf{x})d\mathbf{x}\nonumber\\
	&\leq & p(B_l)e^{-\frac{1}{2}(1-\ln 2)Np(B_l)}+\left[3\frac{e^\frac{\epsilon}{2}+1}{e^\frac{\epsilon}{2}-1}\sqrt{\frac{2\ln (2K)}{Np(B_l)}}+2L^dh^\beta\right]p(B_l),
	\label{eq:b1}
\end{eqnarray}
in which $p(B_l)=\text{P}(\mathbf{X}\in B_l)$ is the probability mass of $B_l$. Moreover, define
\begin{eqnarray}
	\Delta_l=\underset{\mathbf{x}\in B_l}{\inf} \left(\eta^*(\mathbf{x}) - \eta_s(\mathbf{x})\right),
\end{eqnarray}
and
\begin{eqnarray}
	\tilde{\Delta}_l=\underset{\mathbf{x}\in B_l}{\inf} \left(\tilde{\eta}^*(\mathbf{x}) - \tilde{\eta}_s(\mathbf{x})\right)=\frac{e^\frac{\epsilon}{2}-1}{e^\frac{\epsilon}{2}+1}\Delta_l,
\end{eqnarray}
in which the $\tilde{\eta}_s$ is the second largest value of $\tilde{\eta}_j$ among $j=1,\ldots, K$, which follows the definition of $\eta_s$.

If $\Delta_l>0$, then $c^*(\mathbf{x})$ is the same over $B_l$. Then either $v_l^*-v_{lc_l} = 0$ or $v_l^*-v_{lc_l}\geq \Delta_l$ holds. Hence
\begin{eqnarray}
	&&\tilde{\eta}^*(\mathbf{x}) - \mathbb{E}[\tilde{\eta}_{c_l}(\mathbf{x})|\mathbf{X}_{1:N}]\nonumber\\
	&=&\int_0^1 \text{P}\left(\tilde{\eta}^*(\mathbf{x}) - \tilde{\eta}_{c_l}(\mathbf{x})> t|\mathbf{X}_{1:N}\right) dt\nonumber\\
	&\leq & \int_0^1 \text{P}\left(v_l^*-v_{lc_l}>t-2L_dh^\beta\frac{e^\frac{\epsilon}{2}+1}{e^\frac{\epsilon}{2}-1}|\mathbf{X}_{1:N}\right) dt\nonumber\\
	&\leq & \int_0^{\tilde{\Delta}_l+2L_dh^\beta }\text{P}(v_l^*-v_{lc_l}\geq \Delta_l)dt+\int_{\tilde{\Delta}_l+2L_dh^\beta}^\infty 2K\exp\left[-\frac{1}{2}n_l(t-2L_dh^\beta)^2\right] dt\nonumber\\
	&\leq & 2K\exp\left(-\frac{1}{2}n_l\tilde{\Delta}_l^2\right)(\tilde{\Delta}_l+2L_dh^\beta\frac{e^\frac{\epsilon}{2}+1}{e^\frac{\epsilon}{2}-1})+2K\sqrt{\frac{2\pi}{n_l}}\exp\left(-\frac{1}{2}n_l\tilde{\Delta}_l^2\right)\nonumber\\
	&=&\left[2K\left(\tilde{\Delta}_l+2L_dh^\beta\frac{e^\frac{\epsilon}{2}+1}{e^\frac{\epsilon}{2}-1}\right)+2K\sqrt{\frac{2\pi}{n_l}}\right] \exp\left(-\frac{1}{2}n_l\tilde{\Delta}_l^2\right).
	\label{eq:b2}
\end{eqnarray}

Take expectation over $\mathbf{X}_{1:N}$, we get
\begin{eqnarray}
	&&\int_{B_l} (\eta^*(\mathbf{x}) - \mathbb{E}[\eta_{c_l}(\mathbf{x})]) f(\mathbf{x})d\mathbf{x}\leq p(B_l) e^{-\frac{1}{2}(1-\ln 2)Np(B_l)}\nonumber\\
	&&+2Kp(B_l)\left(\Delta_l+2L_dh^\beta+ \frac{e^\frac{\epsilon}{2}+1}{e^\frac{\epsilon}{2}-1}\sqrt{\frac{2\pi}{Np(B_l)}}\right)\exp\left[-\frac{1}{2}Np(B_l)\Delta_l^2\left(\frac{e^\frac{\epsilon}{2}-1}{e^\frac{\epsilon}{2}+1}\right)^2 \right].
	\label{eq:b2i}
\end{eqnarray}

Define
\begin{eqnarray}
	a_l=\left[3\frac{e^\frac{\epsilon}{2}+1}{e^\frac{\epsilon}{2}-1}\sqrt{\frac{2\ln(2K)}{cNh^d}}+2L_dh^\beta\right] p(B_l),
\end{eqnarray}
and
\begin{eqnarray}
	b_l=2Kp(B_l)\left(\Delta_l+2L_dh^\beta+ \frac{e^\frac{\epsilon}{2}+1}{e^\frac{\epsilon}{2}-1}\sqrt{\frac{2\pi}{cNh^d}}\right)\exp\left[-\frac{1}{2}cNh^d\Delta_l^2\left(\frac{e^\frac{\epsilon}{2}-1}{e^\frac{\epsilon}{2}+1}\right)^2 \right].
\end{eqnarray}
From Assumption \ref{ass:multiclass}(c), $p(B_l)\geq cNh^d$. Therefore, from \eqref{eq:b1} and \eqref{eq:b2i}
\begin{eqnarray}
	R-R^*&\leq & \sum_{l=1}^G \left[p(B_l)e^{-\frac{1}{2}(1-\ln 2)Np(B_l)}+\min\{a_l, b_l\}\right]\nonumber\\
	&\leq & e^{-\frac{1}{2}(1-\ln 2)cNh^d} + \sum_{l=1}^G \min\{a_l, b_l\}.
	\label{eq:Rdecomp}
\end{eqnarray}
It remains to bound $\sum_{l=1}^G \min\{a_l, b_l\}$. Note that for all $\mathbf{x}\in B_l$, $\eta^*(\mathbf{x}) - \eta_s(\mathbf{x})\leq \Delta_l+2L_dh^\beta$. Thus
\begin{eqnarray}
	\sum_{l:\Delta_l\leq u} p(B_l)\leq \text{P}\left(\eta^*(\mathbf{X}) - \eta_s(\mathbf{X})\leq u+2L_dh^\beta\right) \leq M(u+2L_dh^\beta)^\gamma.
\end{eqnarray}
Let
\begin{eqnarray}
	\Delta_0=\frac{e^\frac{\epsilon}{2}+1}{e^\frac{\epsilon}{2}-1}\sqrt{\frac{2\ln(2K)}{cNh^d}},
\end{eqnarray}
and
\begin{eqnarray}
	I_0&=&\{l|\Delta_l\leq \Delta_0\},\\
	I_k&=&\{l|2^{k-1}\Delta_0<\Delta_l\leq 2^k\Delta_0\}, k=1,2,\ldots
\end{eqnarray}
Then
\begin{eqnarray}
	\underset{l\in I_0}\min\{a_l, b_l\}&\leq& \sum_{l\in I_0} a_l \nonumber\\
	&\leq & \left(\sum_{l:\Delta_l\leq \Delta_0}p(B_l)\right)\left[3\frac{e^\frac{\epsilon}{2}+1}{e^\frac{\epsilon}{2}-1}\sqrt{\frac{2\ln(2K)}{cNh^d}}+2L_dh^\beta\right]\nonumber\\
	&\leq & M(\Delta_0+2L_dh^\beta)^\gamma \left[3\frac{e^\frac{\epsilon}{2}+1}{e^\frac{\epsilon}{2}-1}\sqrt{\frac{2\ln(2K)}{cNh^d}}+2L_dh^\beta\right]\nonumber\\
	&\lesssim &\left(\frac{1}{\epsilon^2\wedge 1}\frac{\ln K}{Nh^d}\right)^\frac{\gamma+1}{2}+ h^{\beta(\gamma+1)}.
	\label{eq:I0}
\end{eqnarray}

For $I_k$ with $k\geq 1$,
\begin{eqnarray}
	\underset{l\in I_k}\min\{a_l, b_l\}&\leq & \sum_{l\in I_k} b_l\nonumber\\
	&\leq & \left(\sum_{l:\Delta_l\leq 2^k\Delta_0}p(B_l)\right) \cdot 2K\left(2^k \Delta_0+2L_dh^\beta+\Delta_0\right)\exp\left[-\frac{1}{2}\left(\frac{e^\frac{\epsilon}{2}-1}{e^\frac{\epsilon}{2}+1}\right)^2cNh^d 2^{2k-2}\Delta_0^2 \right]\nonumber\\
	&\leq & M(2^k \Delta_0+2L_dh^\beta)^\gamma \left((2^k+1)\Delta_0+2L_dh^\beta\right) (2K)^{-2^{2k-2} + 1}\nonumber\\
	&\leq & M(\Delta_0+2L_dh^\beta)^{\gamma+1} 2^{k\gamma+k-2^{2k-2} + 2}.
\end{eqnarray}
It is obvious that there exists a finite constant $C'<\infty$ that depends on $\gamma$, such that
\begin{eqnarray}
	\sum_{k=1}^\infty 2^{k\gamma+k-2^{2k-2} + 2}\leq C'.
\end{eqnarray}
Therefore
\begin{eqnarray}
	\sum_{k=1}^\infty \sum_{l\in I_k}\min\{a_l, b_l\}\lesssim \left(\frac{1}{\epsilon^2\wedge 1}\frac{\ln K}{Nh^d}\right)^\frac{\gamma+1}{2}+ h^{\beta(\gamma+1)}.
	\label{eq:Ik}
\end{eqnarray}
Combine \eqref{eq:Rdecomp}, \eqref{eq:I0} and \eqref{eq:Ik},
\begin{eqnarray}
	R-R^*\lesssim \left(\frac{1}{\epsilon^2\wedge 1}\frac{\ln K}{Nh^d}\right)^\frac{\gamma+1}{2}+ h^{\beta(\gamma+1)}.
\end{eqnarray}
To minimize the overall excess risk, let
\begin{eqnarray}
	h\sim \left(\frac{N(\epsilon^2 \wedge 1)}{\ln K}\right)^{-\frac{1}{2\beta+d}},
\end{eqnarray}
then
\begin{eqnarray}
	R-R^*\lesssim \left(\frac{N(\epsilon^2 \wedge 1)}{\ln K}\right)^{-\frac{\beta(\gamma+1)}{2\beta+d}}.
\end{eqnarray}
Compare to the simple random response method, the cube splitting avoids the polynomial decrease over $K$.
\section{Proof of Theorem \ref{thm:clslb-cdp}}\label{sec:clslb-cdp}
We still divide the support as the label LDP setting, except that the value of $h$ is different, which will be specified later in this section. Note that \eqref{eq:excessexp} still holds here. Let $\mathbf{V}$ takes values from $\{-1,1\}^m$ randomly with equal probability, and $V_k$ is the $k$-th element. Then $\eta_\mathbf{V}(\mathbf{x})$ is a random function. The corresponding random output of hypothesis testing is denoted as $\hat{V}_k$, which is calculated by \eqref{eq:vkhat}. Then
\begin{eqnarray}
	\underset{\mathcal{A}\in \mathcal{A}_\epsilon}{\inf}\underset{(f,\eta)\in \mathcal{F}_{cls}}{\sup} (R-R^*)&\geq & \frac{1}{2} ch^{\beta+d}\norm{\phi}_1\underset{\mathcal{A}\in \mathcal{A}_\epsilon}{\inf}\underset{\mathbf{v}\in \mathcal{V}}{\max}\sum_{k=1}^m \text{P}(\hat{v}_k\neq v_k)\nonumber\\
	&\geq &\frac{1}{2}h^{\beta+d}\norm{\phi}_1 \underset{\mathcal{A}\in \mathcal{A}_\epsilon}{\inf}\sum_{k=1}^m \text{P}(\hat{V}_k\neq V_k)\nonumber\\
	&=&\frac{1}{2}h^{\beta+d}\norm{\phi}_1 \sum_{k=1}^m \underset{\mathcal{A}\in \mathcal{A}_\epsilon}{\inf} \text{P}(\hat{V}_k\neq V_k),
	\label{eq:mmx1}
\end{eqnarray}
in which the last step holds since $\hat{V}_k$ for different $k$ are calculated independently.

It remains to derive a lower bound of $\text{P}(\hat{V}_k\neq V_k)$. Denote $n_k$ as the number of samples falling in $B_k$, $\bar{Y}_k$ as the average label values in $B_k$:
\begin{eqnarray}
	n_k&:=&\sum_{i=1}^N \mathbf{1}(\mathbf{X}_i\in B_k),\label{eq:nk}\\
	\bar{Y}_k&:=&\frac{1}{n_k}\sum_{i=1}^N Y_i\mathbf{1}(X_i\in B_k).
	\label{eq:ybark}
\end{eqnarray}
Moreover, define
\begin{eqnarray}
	a_k &:=& \frac{1}{n_k}\sum_{i=1}^N |\eta(\mathbf{X}_i)|\mathbf{1}(\mathbf{X}_i\in B_k)\nonumber\\
	&=&\frac{h^\beta}{n_k}\sum_{i=1}^N \phi\left(\frac{\mathbf{X}_i-\mathbf{c}_k}{h}\right)\mathbf{1}(\mathbf{X}_i\in B_k),
	\label{eq:ak}
\end{eqnarray}
in which the last step comes from \eqref{eq:etav}. Then
\begin{eqnarray}
	\mathbb{E}[\bar{Y}_k|\mathbf{X}_{1:N}, V_k]=V_k a_k,
\end{eqnarray}
in which $\mathbf{X}_{1:N}$ means $\mathbf{X}_1,\ldots, \mathbf{X}_N$. We then show the following lemma.
\begin{lem}\label{lem:p}
	If $0\leq t\leq \ln2/(\epsilon n_k)$ and $n_kt$ is an integer, then
	\begin{eqnarray}
		\text{P}(\hat{V}_k=1|\mathbf{X}_{1:N}, \bar{Y}_k=-t)+ \text{P}(\hat{V}_k=-1|\mathbf{X}_{1:N}, \bar{Y}_k=t) \geq \frac{2}{3}.
		\label{eq:p}
	\end{eqnarray}
\end{lem}
\begin{proof}
	Construct $D'$ by changing the label values of $l=n_kt $ items from these $n_k$ samples falling in $B_k$, from $-1$ to $1$. Then the average label values in $B_k$ is denoted as $\bar{Y}_k'$ after such replacement. $\hat{V}_k$ also becomes $\hat{V}_k'$. Then from the $\epsilon$-label CDP requirement,
	\begin{eqnarray}
		\text{P}(\hat{V}_k = 1|\mathbf{X}_{1:N}, \bar{Y}_k=-t) &\overset{(a)}{\geq} & e^{-l\epsilon}\text{P}\left(\hat{V}_k'=1|\mathbf{X}_{1:N}, \bar{Y}_k'=-t+\frac{2l}{n_k}\right)\nonumber\\
		&\overset{(b)}{\geq} & e^{-l\epsilon}\text{P}\left(\hat{V}_k=1|\mathbf{X}_{1:N}, \bar{Y}_k=-t+\frac{2l}{n_k}\right) \nonumber\\
		&\geq & e^{-n_kt\epsilon}\left[1-\text{P}\left(\hat{V}_k=-1|\mathbf{X}_{1:N}, \bar{Y}_k=-t+\frac{2l}{n_k}\right)\right] \nonumber\\
		&\geq &\frac{1}{2}\left[1-\text{P}\left(\hat{V}_k=-1|\mathbf{X}_{1:N}, \bar{Y}_k=t\right)\right].
		\label{eq:lb1}
	\end{eqnarray}
Here (a) uses the group privacy property. The Hamming distance between $D$ and $D'$ is $l$, thus the ratio of probability between $D$ and $D'$ is within $[e^{-l\epsilon}, e^{l\epsilon}]$. (b) holds because the algorithm does not change after changing $D$ to $D'$. Similarly,
	\begin{eqnarray}
		\text{P}(\hat{V}_k = -1|\mathbf{X}_{1:N}, \bar{Y}_k=t)\geq \frac{1}{2}\left[1-\text{P}\left(\hat{V}_k=1|\mathbf{X}_{1:N}, \bar{Y}_k=-t\right)\right].
		\label{eq:lb2}	
	\end{eqnarray}
	Then \eqref{eq:p} can be shown by adding \eqref{eq:lb1} and \eqref{eq:lb2}.
\end{proof}

Now we use Lemma \ref{lem:p} to bound the excess risk. With sufficiently large $n_k$, $\hat{Y}_k$ will be close to Gaussian distribution with mean $a_k$. To be more rigorous, by Berry-Esseen theorem \citep{berry1941accuracy}, for some absolute constant $C_E$,
\begin{eqnarray}
	\text{P}\left(\bar{Y}_k\leq a_k|\mathbf{X}_{1:N}, V_k=1\right) \geq \frac{1}{2}-\frac{C_E}{\sqrt{n_k}}.
	\label{eq:psmally}
\end{eqnarray}
Similarly,
\begin{eqnarray}
	\text{P}\left(\bar{Y}_k\geq -a_k|\mathbf{X}_{1:N}, V_k=-1\right)\geq \frac{1}{2}-\frac{C_E}{\sqrt{n_k}}.
	\label{eq:plargey}
\end{eqnarray}
We first analyze cubes with 
\begin{eqnarray}
	n_k>16C_E^2, a_k<\frac{\ln 2}{\epsilon n_k}.
	\label{eq:acond}
\end{eqnarray}
Under condition \eqref{eq:acond}, the right hand side of \eqref{eq:psmally} and \eqref{eq:plargey} are at least $1/4$. Therefore
\begin{eqnarray}
	\text{P}(\hat{V}_k\neq V_k|\mathbf{X}_{1:N})&=& \frac{1}{2}\text{P}(\hat{V}_k=1|\mathbf{X}_{1:N}, V_k=-1)+\frac{1}{2}\text{P}(\hat{V}_k=-1|\mathbf{X}_{1:N},V_k=1)\nonumber\\
	&\geq &\frac{1}{8}\text{P}\left(\hat{V}_k=1|\mathbf{X}_{1:N},\bar{Y}_k\geq - \frac{\ln 2}{\epsilon n_k}\right) + \frac{1}{8}\text{P}\left(\hat{V}_k=-1|\mathbf{X}_{1:N},\bar{Y}_k\leq \frac{\ln 2}{\epsilon n_k}\right)\nonumber\\
	&\geq & \frac{1}{12}.
\end{eqnarray}
From \eqref{eq:mmx1},
\begin{eqnarray}
	\underset{\mathcal{A}\in \mathcal{A}_\epsilon}{\inf}\underset{(f,\eta)\in \mathcal{F}_{cls}}{\sup} (R-R^*)\geq \frac{1}{2}h^{\beta+d}\norm{\phi}_1\sum_{k=1}^m \frac{1}{12} \text{P}\left(a_k<\frac{\ln 2}{\epsilon n_k}, n_k>16C_E^2 \right).
\end{eqnarray}
From \eqref{eq:phirange}, \eqref{eq:ak} and \eqref{eq:nk}, $a_k\leq h^\beta$. Therefore
\begin{eqnarray}	\underset{\mathcal{A}\in \mathcal{A}_\epsilon}{\inf}\underset{(f,\eta)\in \mathcal{F}_{cls}}{\sup} (R-R^*)\geq \frac{1}{24} h^{\beta+d}\norm{\phi}_1\sum_{k=1}^m \text{P}\left(16C_E^2<n_k<\frac{\ln 2}{\epsilon h^\beta}\right).
	\label{eq:mmx2}
\end{eqnarray}
Recall that each cube has probability mass $ch^d$. Select $h$ such that
\begin{eqnarray}
	2Nch^d= \frac{\ln 2}{\epsilon h^\beta}.
	\label{eq:ub}
\end{eqnarray}
From Chernoff inequality, $16C_E^2<n_k<\ln 2/(\epsilon h^\beta)$ holds with high probability. \eqref{eq:ub} yields
\begin{eqnarray}
	h\sim  (\epsilon N)^{-\frac{1}{d+\beta}}.
	\label{eq:h}
\end{eqnarray}
Recall the bound of $m$ in \eqref{eq:mbound}. Let $m\sim h^{\gamma \beta - d}$, then \eqref{eq:mmx2} becomes
\begin{eqnarray}
	\underset{\mathcal{A}\in \mathcal{A}_\epsilon}{\inf}\underset{(f,\eta)\in \mathcal{F}_{cls}}{\sup} (R-R^*)&\gtrsim& h^{\beta(\gamma+1)}\nonumber\\
	&\gtrsim &  (\epsilon N)^{-\frac{\beta(\gamma+1)}{d+\beta}}.
\end{eqnarray}
Moreover, the standard lower bound for classification \citep{tsybakov2009introduction} is 
\begin{eqnarray}
	\underset{\mathcal{A}\in \mathcal{A}_\epsilon}{\inf}\underset{(f,\eta)\in \mathcal{F}_{cls}}{\sup} (R-R^*) \gtrsim N^{-\frac{\beta(\gamma+1)}{2\beta + d}}.
\end{eqnarray}
Therefore
\begin{eqnarray}
	\underset{\mathcal{A}\in \mathcal{A}_\epsilon}{\inf}\underset{(f,\eta)\in \mathcal{F}_{cls}}{\sup} (R-R^*) &\gtrsim & N^{-\frac{\beta(\gamma+1)}{2\beta + d}}+(\epsilon N)^{-\frac{\beta(\gamma+1)}{d+\beta}}.
\end{eqnarray}
\section{Proof of Theorem \ref{thm:clsub-cdp}}\label{sec:clsub-cdp}
Denote
\begin{eqnarray}
	n_l^*=\max_j n_{lj},
\end{eqnarray}
\begin{eqnarray}
	n_l:=\sum_{j=1}^K n_{lj}=\sum_{i=1}^N \mathbf{1}(\mathbf{X}_i\in B_l).
\end{eqnarray}
For all $j$ such that $n_l^*-n_{lj}>t$,
\begin{eqnarray}
	\text{P}(c_l=j|\mathbf{X}_{1:N}, Y_{1:N})&=&\frac{e^{\epsilon n_{lj}/ 2}}{\sum_{k=1}^K e^{\epsilon n_{lk}/ 2}}\nonumber\\
	&\leq & \frac{e^{\epsilon n_l^*/2}}{\sum_{k=1}^K e^{\epsilon n_{lk}/ 2}} e^{-\frac{1}{2}\epsilon t}\nonumber\\
	&\leq & e^{-\frac{1}{2}\epsilon t}.
\end{eqnarray}
Therefore
\begin{eqnarray}
	\text{P}(n_l^*-n_{lc_l} > t)=\sum_{j:n_l^*-n_{lj}>t} \text{P}(c_l=j|\mathbf{X}_{1:N}, Y_{1:N})\leq Ke^{-\frac{1}{2}\epsilon t}.
\end{eqnarray}
Hence
\begin{eqnarray}
	\mathbb{E}[n_l^*-n_{lc_l}] &=&\int_0^\infty \text{P}(n_l^*-n_{lj} > t)dt\nonumber\\
	&\leq & \int_0^{2\ln K / \epsilon}1 dt+\int_{2\ln K/\epsilon}^\infty Ke^{-\frac{1}{2}\epsilon t} dt\nonumber\\
	&=& \frac{2}{\epsilon}(\ln K + 1).
\end{eqnarray}
Define
\begin{eqnarray}
	v_{lj} = \frac{1}{n_l}\sum_{i=1}^N \mathbf{1}(\mathbf{X}_i\in B_l)\eta_j(\mathbf{X}_i),
\end{eqnarray}
then
\begin{eqnarray}
	\mathbb{E}[n_{lj}|\mathbf{X}_{1:N}]=n_lv_{lj}.
\end{eqnarray}
From Hoeffding's inequality,
\begin{eqnarray}
	\text{P}(|n_{lj} - n_lv_{lj}|>t)\leq 2 e^{-\frac{1}{2n_l} t^2}.
\end{eqnarray}
Thus
\begin{eqnarray}
	\mathbb{E}\left[\max_j |n_{lj}-n_lv_{lj}|\right]&=&\int_0^\infty \text{P}\left(\cup_{j=1}^K \left\{|n_{lj}-n_lv_{lj}|>t \right\}\right) dt\nonumber\\
	&\leq & \int_0^\infty \min\left(1, 2Ke^{-\frac{1}{2n_l}t^2}\right) dt\nonumber\\
	&=& \sqrt{2n_l\ln(2K)}+\int_{\sqrt{2n_l\ln (2K)}}^\infty 2Ke^{-\frac{1}{2n_l} t^2} dt\nonumber\\
	&<& 2\sqrt{2n_l\ln(2K)},
\end{eqnarray}
in which the last step uses the inequality $\int_t^\infty e^{-u^2/(2\sigma^2)} du\leq \sqrt{2\pi} \sigma e^{-t^2/(2\sigma^2)}$. Then
\begin{eqnarray}
	\mathbb{E}[v_l^*-v_{lc_l}|\mathbf{X}_{1:N}]&=& \frac{1}{n_l}\mathbb{E}[n_lv_l^*-n_lv_{lc_l}]\nonumber\\
	&=&\frac{1}{n_l}\mathbb{E}\left[n_l^*-n_{lc_l}+n_lv_l^*-n_l^*+n_{lc_l}-n_lv_{lc_l}\right]\nonumber\\
	&\leq & \frac{1}{n_l}\mathbb{E}[n_l^*-n_{lc_l}]+\frac{2}{n_l}\mathbb{E}\left[\max_j |n_{lj}-n_lv_{lj}|\right]\nonumber\\
	&\leq & \frac{2}{\epsilon n_l}(\ln K + 1)+4\sqrt{\frac{2\ln(2K)}{n_l}}.
\end{eqnarray}
By H{\"o}lder continuity assumption (Assumption \ref{ass:multiclass}(a)), for $\mathbf{x}\in B_l$, 
\begin{eqnarray}
	|v_{lj}-\eta_j(\mathbf{x})|\leq \frac{1}{n_l}\sum_{i=1}^N \mathbf{1}(\mathbf{X}_i\in B_l)|\eta_j(\mathbf{X}_i)-\eta_j(\mathbf{x})|\leq L_dh^\beta,
\end{eqnarray}
in which $L_d=L\sqrt{d}$, $L$ is the constant in Assumption \ref{ass:multiclass}(a). Thus
\begin{eqnarray}
	\mathbb{E}[\eta^*(\mathbf{x})- \eta_{c_l}(\mathbf{x})|\mathbf{X}_{1:N}]\leq \frac{2}{\epsilon n_l}(\ln K + 1)+4\sqrt{\frac{2\ln(2K)}{n_l}}+2L_dh^\beta.
\end{eqnarray}
Now take integration over $B_l$.
\begin{eqnarray}
	&&\int_{B_l} \left(\eta^*(\mathbf{x})-\mathbb{E}[\eta_{c_l}(\mathbf{x})]\right) f(\mathbf{x})d\mathbf{x}\nonumber\\
	&\leq & \text{P}\left(n_l<\frac{1}{2} Np(B_l)\right) \int_{B_l} \left(\eta^*(\mathbf{x}) - \mathbb{E}\left[\eta_{c_l}(\mathbf{x}) |n_l<\frac{1}{2}Np(B_l)\right]\right) f(\mathbf{x})d\mathbf{x}\nonumber\\
	&& + \int_{B_l} \left(\eta^*(\mathbf{x})- \mathbb{E}\left[\eta_{c_l}(\mathbf{x})|n_l\geq \frac{1}{2}Np(B_l)\right]\right) f(\mathbf{x})d\mathbf{x}\nonumber\\
	&\leq & p(B_l)\exp\left[-\frac{1}{2}(1-\ln 2)Np(B_l)\right] + \left[\frac{2(\ln K+1)}{\epsilon Np(B_l)}+4\sqrt{\frac{2\ln(2K)}{Np(B_l)}}+2L_dh^\beta\right] p(B_l),\nonumber\\
	\label{eq:b1cdp}
\end{eqnarray}
in which $p(B_l)=\text{P}(\mathbf{X}\in B_l)=\int_{B_l} f(\mathbf{x})d\mathbf{x}$. \eqref{eq:b1cdp} is the label CDP counterpart of \eqref{eq:b1}. The remainder of the proof follows arguments of the label LDP. We omit detailed steps. The result is
\begin{eqnarray}
	R-R^*\lesssim \left(\frac{\ln K}{\epsilon N h^d}+\sqrt{\frac{\ln K}{Nh^d}}+h^\beta\right)^{\gamma + 1}.
\end{eqnarray}
Let
\begin{eqnarray}
	h\sim \left(\frac{\ln K}{\epsilon N}\right)^\frac{1}{\beta+d}+\left(\frac{\ln K}{N}\right)^\frac{1}{2\beta + d},
\end{eqnarray}
then
\begin{eqnarray}
	R-R^*\lesssim \left(\frac{\ln K}{\epsilon N}\right)^\frac{\beta(\gamma + 1)}{\beta + d}+\left(\frac{\ln K}{N}\right)^\frac{\beta(\gamma + 1)}{2\beta + d}.
\end{eqnarray}
The proof is complete.

\section{Proof of Theorem \ref{thm:reglb}}\label{sec:reglb}
Find $G$ cubes in the support and the length of each cube is $h$. Let $\phi(\mathbf{u})$ be the same as the classification case shown in appendix \ref{sec:clslb}. For $\mathbf{v}\in \mathcal{V}:=\{-1,1\}^G$, let
\begin{eqnarray}
	\eta_\mathbf{v}(\mathbf{x}) = \sum_{k=1}^K v_k\phi\left(\frac{\mathbf{x}-\mathbf{c}_k}{h}\right) h^\beta.
	\label{eq:etavdf}
\end{eqnarray}
Let $\text{P}(Y=1|\mathbf{x}) = (1+\eta_\mathbf{v}(\mathbf{x}))/2$, $\text{P}(Y=-1|\mathbf{x})=(1-\eta_\mathbf{v}(\mathbf{x}))$, then $\eta(\mathbf{x})=\mathbb{E}[Y|\mathbf{x}] = \eta_\mathbf{v}(\mathbf{x})$.

The overall volume of the support is bounded. Thus, we have
\begin{eqnarray}
	G\leq C_G h^{-d}
\end{eqnarray}
for some constant $C_G$. 

Denote
\begin{eqnarray}
	\hat{v}_k=\sign\left(\int_{B_k}\hat{\eta}(\mathbf{x})\phi\left(\frac{\mathbf{x}-\mathbf{c}_k}{h}\right) f(\mathbf{x})d\mathbf{x}\right),
	\label{eq:vkreg}
\end{eqnarray}
then the excess risk is bounded by
\begin{eqnarray}
	R &=& \mathbb{E}\left[\left(\hat{\eta}(\mathbf{X}) - \eta_\mathbf{v}(\mathbf{X})\right)^2 \right]\nonumber\\
	&=&\sum_{k=1}^K \int_{B_k} \mathbb{E}\left[(\hat{\eta}(\mathbf{x}) - \eta_\mathbf{v}(\mathbf{x}))^2 \right] f(\mathbf{x})d\mathbf{x}.
\end{eqnarray}
If $\hat{v}_k\neq v_k$, from \eqref{eq:vkreg},
\begin{eqnarray}
	\int_{B_k} \left(\hat{\eta}(\mathbf{x}) - v_k\phi\left(\frac{\mathbf{x}-\mathbf{c}_k}{h}\right) h^\beta\right)^2 f(\mathbf{x})d\mathbf{x}\geq \int_{B_k} \left(\hat{\eta}(\mathbf{x}) + v_k\phi\left(\frac{\mathbf{x}-\mathbf{c}_k}{h}\right) h^\beta \right)^2 f(\mathbf{x})d\mathbf{x}.
\end{eqnarray}
Therefore, if $\hat{v}_k\neq v_k$, then
\begin{eqnarray}
	\int_{B_k} \left(\hat{\eta}(\mathbf{x}) -\eta_\mathbf{v}(\mathbf{x})\right)^2 d\mathbf{x}\geq \frac{1}{2}\int_{B_k}\phi^2 \left(\frac{\mathbf{x}-\mathbf{c}_k}{h}\right) h^{2\beta} f(\mathbf{x})d\mathbf{x}=\frac{1}{2} ch^{2\beta+d} \norm{\phi}_2^2.
\end{eqnarray}
Therefore
\begin{eqnarray}
	R - R^*&\geq & \mathbb{E}\left[\frac{1}{2} ch^{2\beta+d}\norm{\phi}_2^2\mathbf{1}(\hat{v}_k\neq v_k)\right]\nonumber\\
	&=&\frac{1}{2} ch^{2\beta+d}\norm{\phi}_2^2 \mathbb{E}[\rho_H(\hat{\mathbf{v}}, \mathbf{v})].
	\label{eq:Rexp}
\end{eqnarray}
Similar to the classification problem analyzed in Appendix \ref{sec:clslb}, let
\begin{eqnarray}
	h\sim \left(N(\epsilon\wedge 1)^2\right)^{-\frac{1}{2\beta+d}},
\end{eqnarray}
then $\delta \lesssim 1$, and
\begin{eqnarray}
	\underset{\hat{\mathbf{v}}}{\inf}\underset{M\in \mathcal{M}_\epsilon}{\sup}\underset{\mathbf{v}\in \mathcal{V}}{\max}\mathbb{E}[\rho_H(\hat{\mathbf{v}}, \mathbf{v})]\gtrsim G\sim h^{-d}.
\end{eqnarray}
Thus
\begin{eqnarray}
	\underset{\hat{\eta}}{\inf}\underset{M\in \mathcal{M}_\epsilon}{\inf}\underset{P_{X,Y}\in \mathcal{F}_{reg1}}{\sup} R\gtrsim h^{2\eta+d} h^{-d}\sim h^{2\beta} \sim (N(\epsilon\wedge 1)^2)^{-\frac{2\beta}{2\beta+d}}.
\end{eqnarray}
\section{Proof of Theorem \ref{thm:regbounded}}\label{sec:regbounded}
According to Assumption \ref{ass:bounded}, $|Y|<T$ with probability $1$, thus $\Var[Y|\mathbf{x}]\leq T^2$ for any $\mathbf{x}$. A Laplacian distribution with parameter $\lambda$ has variance $2\lambda^2$, thus
\begin{eqnarray}
	\Var[W]=2\lambda^2 =2\left(\frac{2T}{\epsilon}\right)^2=\frac{8T^2}{\epsilon^2}.
\end{eqnarray}
Hence
\begin{eqnarray}
	\Var[Z]=\Var[Y]+\Var[W]\leq T^2\left(1+\frac{8}{\epsilon^2}\right).
\end{eqnarray}
Now we analyze the bias first.
\begin{eqnarray}
	\mathbb{E}[\hat{\eta}(\mathbf{x})]=\mathbb{E}\left[\frac{1}{k}\sum_{i\in \mathcal{N}_k(\mathbf{x})}Z_i\right] = \mathbb{E}\left[\frac{1}{k}\sum_{i\in \mathcal{N}_k(\mathbf{x})}\eta(\mathbf{X}_i)\right].
\end{eqnarray}
Thus
\begin{eqnarray}
	|\mathbb{E}[\hat{\eta}(\mathbf{x})] - \eta(\mathbf{x})|&\leq & \mathbb{E}\left[\frac{1}{k}\sum_{i\in \mathcal{N}_k(\mathbf{x})}|\eta(\mathbf{X}_i)-\eta(\mathbf{x})|\right]\nonumber\\
	&\leq & \mathbb{E}\left[\frac{1}{k}\sum_{i\in \mathcal{N}_k(\mathbf{x})}\min\left\{L\norm{\mathbf{X}_i-\mathbf{x}}^\beta, 2T \right\}\right]\nonumber\\
	&\leq & \mathbb{E}\left[\frac{1}{k}\sum_{i\in \mathcal{N}_k(\mathbf{x})}\min\left\{L\rho^\beta(\mathbf{x}), 2T \right\}\right]\nonumber\\
	&\leq & 2T\text{P}(\rho(\mathbf{x})> r_0) + Lr_0^\beta\nonumber\\
	&\leq & 2Te^{-(1-\ln 2)k}+L\left(\frac{2k}{Ncv_d\theta}\right)^\frac{\beta}{d}\nonumber\\
	&\leq & C_1\left(\frac{k}{N}\right)^\frac{\beta}{d},
	\label{eq:biasreg}
\end{eqnarray}
for some constant $C_1$.

It remains to bound the variance.
\begin{eqnarray}
	\Var[\hat{\eta}(\mathbf{x})] = \mathbb{E}\left[\Var\left[\hat{\eta}(\mathbf{x})|\mathbf{X}_1,\ldots, \mathbf{X}_N\right]\right]+\Var[\mathbb{E}[\hat{\eta}(\mathbf{x})]|\mathbf{X}_1,\ldots, \mathbf{X}_N].
	\label{eq:vardecomp}
\end{eqnarray}
For the first term in \eqref{eq:vardecomp},
\begin{eqnarray}
	\Var[\hat{\eta}(\mathbf{x})|\mathbf{X}_1,\ldots, \mathbf{X}_N]&=&\Var\left[\frac{1}{k}\sum_{i\in \mathcal{N}_k(\mathbf{x})} Z_i|\mathbf{X}_1,\ldots, \mathbf{X}_N\right]\nonumber\\
	&=&\frac{1}{k^2}\sum_{i\in \mathcal{N}_k(\mathbf{x})}\Var[Z_i|\mathbf{X}_1,\ldots, \mathbf{X}_N]\nonumber\\
	&\leq & \frac{1}{k} T^2\left(1+\frac{8}{\epsilon^2}\right).
\end{eqnarray}
For the second term in \eqref{eq:vardecomp},
\begin{eqnarray}
	\Var[\mathbb{E}[\hat{\eta}(\mathbf{x})|\mathbf{X}_1,\ldots, \mathbf{X}_N]]&=&\Var\left[\frac{1}{k}\sum_{i\in \mathcal{N}_k(\mathbf{x})} \eta(\mathbf{X}_i)\right]\nonumber\\
	&\leq & \mathbb{E}\left[\left(\frac{1}{k}\sum_{i\in \mathcal{N}_k(\mathbf{x})} \eta(\mathbf{X}_i) - \eta(\mathbf{x})\right)^2\right]\nonumber\\
	&= &\frac{1}{k} \sum_{i\in \mathcal{N}_k(\mathbf{x})} \mathbb{E}\left[(\eta(\mathbf{X}_i) - \eta(\mathbf{x}))^2\right]\nonumber\\
	&\leq & \frac{1}{k}\sum_{i\in \mathcal{N}_k(\mathbf{x})} \mathbb{E}\left[\min\left\{L^2 \norm{\mathbf{X}_i-\mathbf{x}}^{2\beta}, 4T^2 \right\}\right]\nonumber\\
	&\leq & 4T^2 e^{-(1-\ln 2)k}+L^2 r_0^{2\beta}\nonumber\\
	&\leq & C_1^2 \left(\frac{k}{N}\right)^\frac{2\beta}{d}.
	\label{eq:vare}
\end{eqnarray}
Therefore \eqref{eq:vardecomp} becomes
\begin{eqnarray}
	\Var[\hat{\eta}(\mathbf{x})]\leq \frac{1}{k}T^2\left(1+\frac{8}{\epsilon^2}\right) + C_1^2 \left(\frac{k}{N}\right)^\frac{2\beta}{d}.
\end{eqnarray}
Combine the analysis of bias and variance,
\begin{eqnarray}
	\mathbb{E}[(\hat{\eta}(\mathbf{x}) - \eta(\mathbf{x}))^2]\leq \frac{1}{k} T^2 \left(1+\frac{8}{\epsilon^2}\right) + 2C_1^2 \left(\frac{k}{N}\right)^\frac{2\beta}{d}.
\end{eqnarray}
Therefore the overall risk is bounded by
\begin{eqnarray}
	R=\mathbb{E}[(\hat{\eta}(\mathbf{X}) - \eta(\mathbf{X}))^2]\lesssim \frac{1}{k} T^2 \left(1+\frac{8}{\epsilon^2}\right) + 2C_1^2 \left(\frac{k}{N}\right)^\frac{2\beta}{d}.
\end{eqnarray}

The optimal growth rate of $k$ over $N$ is
\begin{eqnarray}
	k\sim N^\frac{2\beta}{d+2\beta} (\epsilon\wedge 1)^{-\frac{2d}{d+2\beta}}.
\end{eqnarray}
Then the convergence rate of the overall risk becomes
\begin{eqnarray}
	R\lesssim (N(\epsilon\wedge 1)^2)^{-\frac{2\beta}{d+2\beta}}.
\end{eqnarray}

\section{Proof of Theorem \ref{thm:reglbcdp}}\label{sec:reglbcdp}

From \eqref{eq:Rexp},
\begin{eqnarray}
	R-R^*&\geq& \frac{1}{2}ch^{2\beta + d}\norm{\phi}_2^2 \mathbb{E}[\rho_H(\hat{\mathbf{V}}, \mathbf{V})]\nonumber\\
	&=&\frac{1}{2}ch^{2\beta+d}\norm{\phi}_2^2\sum_{k=1}^G \text{P}(\hat{V}_k\neq V_k).
\end{eqnarray}
Follow the analysis of lower bounds of classification in Appendix \ref{sec:clslb-cdp}, let $h$ scales as \eqref{eq:h}, then $\text{P}(\hat{V}_k\neq V_k)\gtrsim 1$. Moreover, $G\sim h^{-d}$. Hence
\begin{eqnarray}
	\underset{\mathcal{A}\in \mathcal{A}_\epsilon}{\inf} \underset{(f, \eta)\in \mathcal{F}_{reg1}}{\sup} (R-R^*)\gtrsim h^{2\beta}\sim  (\epsilon N)^{-\frac{2\beta}{d+\beta}}.
	\label{eq:regmmx1}
\end{eqnarray}
Moreover, note that the non-private lower bound of regression is 
\begin{eqnarray}
	\underset{\mathcal{A}\in \mathcal{A}_\epsilon}{\inf} \underset{(f, \eta)\in \mathcal{F}_{reg1}}{\sup} (R-R^*)\gtrsim N^{-\frac{2\beta}{2\beta+d}}.
	\label{eq:regmmx2}
\end{eqnarray}
Combine \eqref{eq:regmmx1} and \eqref{eq:regmmx2},
\begin{eqnarray}
	\underset{\mathcal{A}\in \mathcal{A}_\epsilon}{\inf} \underset{(f, \eta)\in \mathcal{F}_{reg1}}{\sup} (R-R^*)\gtrsim N^{-\frac{2\beta}{2\beta+d}}+(\epsilon N)^{-\frac{2\beta}{d+\beta}}.
\end{eqnarray}
\section{Proof of Theorem \ref{thm:regubcdp}}\label{sec:regubcdp}

\emph{1) Analysis of bias.} Note that
\begin{eqnarray}
	\mathbb{E}[\hat{\eta}_l|\mathbf{X}_{1:N}]=\mathbb{E}[Y|\mathbf{X}\in B_l]=\frac{1}{p(B_l)}\int \eta(\mathbf{u})f(\mathbf{u})d\mathbf{u}.
\end{eqnarray}
Therefore, for all $\mathbf{x}\in B_l$, 
\begin{eqnarray}
	\left|\mathbb{E}[\hat{\eta}_l|\mathbf{X}_{1:N}]-\eta(\mathbf{x})\right|&\leq& \frac{1}{p(B_l)}\int |\eta(\mathbf{u})-\eta(\mathbf{x})|f(\mathbf{u})d\mathbf{u}\nonumber\\
	&\leq & L_dh^\beta.	
	\label{eq:avgbias}
\end{eqnarray}
Therefore for all $\mathbf{x}\in B_l$,
\begin{eqnarray}
	|\mathbb{E}[\hat{\eta}_l] - \eta(\mathbf{x})|\leq L_dh^\beta.
	\label{eq:bias}
\end{eqnarray}
\emph{2) Analysis of variance.} If $n_l>0$,
\begin{eqnarray}
	\Var\left[\frac{1}{n_l}\sum_{i=1}^N \mathbf{1}(\mathbf{X}_i\in B_l)Y_i|\mathbf{X}_{1:N}\right] =\frac{1}{n_l}\Var[Y|\mathbf{X}\in B_l]\leq \frac{1}{n_l}.
\end{eqnarray}
Therefore
\begin{eqnarray}
	\Var\left[\frac{1}{n_l}\sum_{i=1}^N \mathbf{1}(\mathbf{X}_i\in B_l)Y_i\right] &\leq & \text{P}\left(n_l<\frac{1}{2}Np(B_l)\right) + \text{P}\left(n_l\geq \frac{1}{2}Np(B_l)\right)\frac{2}{Np(B_l)}\nonumber\\
	&\leq & \exp\left[-\frac{1}{2}(1-\ln 2)Np(B_l)\right] + \frac{2}{Nch^d}.
	\label{eq:varmeany}
\end{eqnarray}
Similarly,
\begin{eqnarray}
	\Var[W_l]&\leq& \text{P}\left(n_l<\frac{1}{2}Np(B_l)\right) \frac{1}{\epsilon^2}+ \text{P}\left(n_l\geq \frac{1}{2} Np(B_l)\right) \frac{8}{\left(\frac{1}{2}Np(B_l)\right)^2\epsilon^2}\nonumber\\
	&\lesssim & \frac{1}{N^2h^{2d}\epsilon^2}.
\end{eqnarray}
The mean squared error can then be bounded by the bounds of bias and variance.
\begin{eqnarray}
	\mathbb{E}\left[(\hat{\eta}(\mathbf{x}) - \eta(\mathbf{x}))^2\right]\lesssim h^{2\beta} + \frac{1}{Nh^d}+\frac{1}{N^2h^{2d}\epsilon^2}.
	\label{eq:msereg}
\end{eqnarray}
Let
\begin{eqnarray}
	h\sim N^{-\frac{1}{2\beta+d}}+(\epsilon N)^{-\frac{1}{d+\beta}}.
\end{eqnarray}
Then
\begin{eqnarray}
	R-R^*\lesssim N^{-\frac{2\beta}{2\beta+d}}+(\epsilon N)^{-\frac{2\beta}{d+\beta}}.
	\label{eq:regfinal}
\end{eqnarray}

\section{Proof of Theorem \ref{thm:taillb}}

Now we prove the minimax lower bound of nonparametric regression under label CDP constraint. We focus on the case in which $\epsilon$ is small.

Similar to the steps of the proof of Theorem \ref{thm:reglb} in Appendix \ref{sec:reglb}, we find $B$ cubes in the support. The definition of $\eta_\mathbf{v}$, $\hat{v}_k$ are also the same as \eqref{eq:etavdf} and \eqref{eq:vkreg}. Compared with the case with bounded noise, now $Y$ can take values in $\mathbb{R}$. 

For given $\mathbf{x}$, let
\begin{eqnarray}
	Y=\left\{
	\begin{array}{ccc}
		T &\text{with probability} & \frac{1}{2}\left(\frac{M_p}{T^p}+\frac{\eta_\mathbf{v}(\mathbf{x})}{T}\right)\\
		0 &\text{with probability} & 1-\frac{M_p}{T^p}\\
		-T &\text{with probability} & \frac{1}{2}\left(\frac{M_p}{T^p} - \frac{\eta_\mathbf{v}(\mathbf{x})}{T}\right).
	\end{array}
	\right.
	\label{eq:Ydist}
\end{eqnarray}
In Appendix \ref{sec:reglb} about the case with bounded noise, $T$ is a fixed constant. However, here $T$ is not fixed and will change over $N$. It is straightforward to show that the distribution of $Y$ in \eqref{eq:Ydist} satisfies Assumption \ref{ass:tail}:
\begin{eqnarray}
	\mathbb{E}[|Y|^p|\mathbf{x}] = M_p.
\end{eqnarray}
Moreover, by taking expectation over $Y$, it can be shown that $\eta_\mathbf{v}$ is still the regression function:
\begin{eqnarray}
	\mathbb{E}[Y|\mathbf{x}] = \eta_\mathbf{v}(\mathbf{x}).
\end{eqnarray}
Let
\begin{eqnarray}
	T=\left(\frac{1}{2}M_ph^{-\beta}\right)^\frac{1}{p-1}.
\end{eqnarray}
Here we still define
\begin{eqnarray}
	\delta =\underset{M\in \mathcal{M}_\epsilon}{\sup}\underset{\mathbf{v}, \mathbf{v}':\rho_H(\mathbf{v}, \mathbf{v}')=1}{\max}D(P_{(X,Z)_{1:N}|\mathbf{v}}||P_{(X,Z)_{1:N}|\mathbf{v}'}).
\end{eqnarray}
Without loss of generality, suppose that $v_1= v_1'$ for $i\neq 1$. Then
\begin{eqnarray}
	D(P_{(X,Z)_{1:N}|\mathbf{v}}||P_{(X,Z)_{1:N}|\mathbf{v}'}) &=& ND(P_{X,Z|\mathbf{v}}|| P_{X,Z|\mathbf{v}'})\nonumber\\
	&=&N\int_{B_1} f(\mathbf{x}) D(P_{Z|\mathbf{X}, \mathbf{v}}||P_{Z|\mathbf{X}, \mathbf{v}'})d\mathbf{x}\nonumber\\
	&\leq & N\int_{B_1} f(\mathbf{x})(e^\epsilon-1)^2 \mathbb{TV}^2 \left(P_{Z|X, \mathbf{v}}, P_{Z|X, \mathbf{v}'}\right) d\mathbf{x}\nonumber\\
	&=& N\int_{B_1} f(\mathbf{x}) (e^\epsilon - 1)^2 \eta_\mathbf{v}^2(\mathbf{x}) \frac{1}{T^2}d\mathbf{x}\nonumber\\
	&=&N(e^\epsilon-1)^2 \frac{h^{2\beta}}{T^2}\int_{B_1} f(\mathbf{x})\phi^2\left(\frac{\mathbf{x}-\mathbf{c}_1}{h}\right) d\mathbf{x}\nonumber\\
	&=& N(e^\epsilon-1)^2 h^{2\beta+d}\norm{\phi}_2^2T^{-2}\nonumber\\
	&=& N(e^\epsilon-1)^2 \norm{\phi}_2^2\left(\frac{1}{2}M_p\right)^{-\frac{2}{p-1}}h^{2\beta+d+\frac{2\beta}{p-1}}.
\end{eqnarray}
Let
\begin{eqnarray}
	h \sim (N(e^\epsilon - 1)^2)^{-\frac{p-1}{2p\beta+d(p-1)}},
\end{eqnarray}
then $\delta\lesssim 1$. Hence
\begin{eqnarray}
	\underset{\hat{\eta}}{\inf}\underset{M\in \mathcal{M}_\epsilon}{\inf}\underset{(f,\eta)\in \mathcal{F}}{\sup}R\gtrsim h^{2\beta}\sim (N(e^\epsilon - 1)^2)^{-\frac{2\beta(p-1)}{2p\beta+d(p-1)}}.
\end{eqnarray}
\section{Proof of Theorem \ref{thm:tail}}\label{sec:tail}
Define
\begin{eqnarray}
	\eta_T(\mathbf{x}):=\mathbb{E}[Y_T|\mathbf{x}].
\end{eqnarray}
Then
\begin{eqnarray}
	\hat{\eta}(\mathbf{x}) - \eta(\mathbf{x}) = \eta_T(\mathbf{x}) - \eta(\mathbf{x}) +\mathbb{E}[\hat{\eta}(\mathbf{x})] - \eta_T(\mathbf{x}) + \hat{\eta}(\mathbf{x}) - \mathbb{E}[\hat{\eta}(\mathbf{x})].
\end{eqnarray}
Therefore
\begin{eqnarray}
	\mathbb{E}\left[\left(\hat{\eta}(\mathbf{x}) - \eta(\mathbf{x})\right)^2\right]&\leq &3(\eta_T(\mathbf{x}) - \eta(\mathbf{x}))^2 + 3(\mathbb{E}[\hat{\eta}(\mathbf{x})] - \eta_T(\mathbf{x}))^2+3\Var[\hat{\eta}(\mathbf{x})]\nonumber\\
	&:=& 3(I_1+I_2+I_3).
	\label{eq:mse}
\end{eqnarray}
Now we bound $I_1$, $I_2$ and $I_3$ separately.

\noindent \textbf{Bound of $I_1$.} We show the following lemma (which will also be used later).
\begin{lem}\label{lem:clipbias}
	\begin{eqnarray}
		|\eta_T(\mathbf{x})-\eta(\mathbf{x})|\leq \frac{M_p}{p-1}T^{1-p}.	
		\label{eq:clipbias}
	\end{eqnarray}
\end{lem}
\begin{proof}
	Firstly, we decompose $\eta_T(\mathbf{x})$ and $\eta(\mathbf{x})$:
	\begin{eqnarray}
		\eta_T(\mathbf{x})=\mathbb{E}[Y_T|\mathbf{x}] = \mathbb{E}[Y\mathbf{1}(-T\leq Y\leq T)|\mathbf{x}]+T\text{P}(Y>T|\mathbf{x})-TP(Y<T|\mathbf{x}),
		\label{eq:etat}
	\end{eqnarray}
	\begin{eqnarray}
		\eta(\mathbf{x})=\mathbb{E}[Y|\mathbf{x}] = \mathbb{E}[Y\mathbf{1}(-T\leq Y\leq T)|\mathbf{x}]+\mathbb{E}[Y\mathbf{1}(Y>T)|\mathbf{x}] - \mathbb{E}[Y\mathbf{1}(Y<T)|\mathbf{x}].
		\label{eq:eta}
	\end{eqnarray}
	The first term is the same between \eqref{eq:etat} and \eqref{eq:eta}. Therefore we only need to compare the second and the third term.
	\begin{eqnarray}
		\mathbb{E}[Y\mathbf{1}(Y>T)|\mathbf{x}] &=&\int_0^T \text{P}(Y>T|\mathbf{x})dt+\int_T^\infty \text{P}(Y>T|\mathbf{x})dt\nonumber\\
		&\leq & T\text{P}(Y>T|\mathbf{x}) + \int_T^\infty M_pt^{-p}dt\nonumber\\
		&=& T\text{P}(Y>T|\mathbf{x}) + \frac{M_p}{p-1}T^{1-p}.
	\end{eqnarray}
	Therefore
	\begin{eqnarray}
		\mathbb{E}[Y\mathbf{1}(Y>T)|\mathbf{x}]-T\text{P}(Y>T|\mathbf{x})\leq \frac{M_p}{p-1}T^{1-p}.
	\end{eqnarray}
	Similarly,
	\begin{eqnarray}
		TP(Y<T|\mathbf{x})-\mathbb{E}[Y\mathbf{1}(Y<T)|\mathbf{x}]\leq \frac{M_p}{p-1}T^{1-p}.
	\end{eqnarray}
	Combining these two inequalities yields \eqref{eq:clipbias}.
\end{proof}
With Lemma \ref{lem:clipbias},
\begin{eqnarray}
	I_1\leq \frac{M_p^2}{(p-1)^2}T^{2(1-p)}.
	\label{eq:I1bd}
\end{eqnarray}	

\noindent \textbf{Bound of $I_2$.} Follow the steps in \eqref{eq:biasreg},
\begin{eqnarray}
	I_2\leq C_1^2\left(\frac{k}{N}\right)^\frac{2\beta}{d}.
	\label{eq:I2bd}
\end{eqnarray}

\noindent \textbf{Bound of $I_3$.} We decompose $\Var[\hat{\eta}(\mathbf{x})]$ as following:
\begin{eqnarray}
	\Var[\hat{\eta}(\mathbf{x})] = \mathbb{E}[\Var[\hat{\eta}(\mathbf{x})|\mathbf{X}_1,\ldots, \mathbf{X}_N]]+\Var[\mathbb{E}[\hat{\eta}(\mathbf{x})|\mathbf{X}_1,\ldots, \mathbf{X}_N]].
	\label{eq:varsplit}
\end{eqnarray}
For the first term in \eqref{eq:varsplit}, from Assumption \ref{ass:tail}, $\mathbb{E}[|Y|^p|\mathbf{x}]\leq M_p$. Since $p\geq 2$, we have $\mathbb{E}[Y^2|\mathbf{x}]=M_p^\frac{2}{p}$. Therefore
\begin{eqnarray}
	\Var[Z_i|\mathbf{X}_1,\ldots, \mathbf{X}_N]=\Var[Y_T]+\Var[W]\leq M_p^\frac{2}{p}+\frac{8T^2}{\epsilon^2}.
\end{eqnarray}
Recall \eqref{eq:regest}, we have
\begin{eqnarray}
	\Var[\hat{\eta}(\mathbf{x})|\mathbf{X}_1,\ldots, \mathbf{X}_N]&=&\frac{1}{k^2}\sum_{i\in \mathcal{N}_k(\mathbf{x})}\Var[Z_i|\mathbf{X}_1,\ldots, \mathbf{X}_N]\nonumber\\
	&\leq & \frac{1}{k}\left(M_p^\frac{2}{p}+\frac{8T^2}{\epsilon^2}\right).
\end{eqnarray}
For the second term in \eqref{eq:varsplit}, \eqref{eq:vare} still holds, thus
\begin{eqnarray}
	\Var[\mathbb{E}[\hat{\eta}(\mathbf{x})|\mathbf{X}_1,\ldots, \mathbf{X}_N]]\leq C_1^2\left(\frac{k}{N}\right)^\frac{2\beta}{d},	
\end{eqnarray}
and
\begin{eqnarray}
	I_3\leq \frac{1}{k}\left(M_p^\frac{2}{p}+\frac{8T^2}{\epsilon^2}\right)+C_1^2\left(\frac{k}{N}\right)^\frac{2\beta}{d}.
	\label{eq:I3bd}
\end{eqnarray}
Plug \eqref{eq:I1bd}, \eqref{eq:I2bd} and \eqref{eq:I3bd} into \eqref{eq:mse}, and take expectations, we get
\begin{eqnarray}
	R&=&\mathbb{E}[(\hat{\eta}(\mathbf{X}) - \eta(\mathbf{X}))^2]\nonumber\\
	&\lesssim & T^{2(1-p)}+\frac{1}{k}+\frac{T^2}{k\epsilon^2}+\left(\frac{k}{N}\right)^\frac{2\beta}{d}.
\end{eqnarray}
Let 
\begin{eqnarray}
	T\sim (k\epsilon^2)^\frac{1}{2p},
	k\sim (N\epsilon^2)^\frac{2p\beta}{d(p-1)+2p\beta}\vee N^\frac{2\beta}{2\beta+d},
\end{eqnarray}
then
\begin{eqnarray}
	R\lesssim (N\epsilon^2)^{-\frac{2\beta(p-1)}{d(p-1)+2p\beta}}\vee N^{-\frac{2\beta}{2\beta+d}}.
\end{eqnarray}
\section{Proof of Theorem \ref{thm:taillbcdp}}\label{sec:taillbcdp}
Let $Y$ be distributed as \eqref{eq:Ydist}. Recall Lemma \ref{lem:p} for the problem of classification and regression with bounded noise.

Now we show the corresponding lemma for regression with unbounded noise.
\begin{lem}\label{lem:pnew}
	If $0\leq t\leq T\ln2/(\epsilon n_k)$ and $n_kt/T$ is an integer, then
	\begin{eqnarray}
		\text{P}(\hat{V}_k=1|\mathbf{X}_{1:N}, \bar{Y}_k=-t)+ \text{P}(\hat{V}_k=-1|\mathbf{X}_{1:N}, \bar{Y}_k=t) \geq \frac{2}{3}.
		\label{eq:pnew}
	\end{eqnarray}
\end{lem}
Here we briefly explain the condition $n_kt/T$ is an integer. Recall the definition of $\bar{Y}_k$ in \eqref{eq:ybark}. Now since $Y$ take values in $\{-T,0,T\}$, $n_k\bar{Y}_k/T$ must be an integer. Therefore, in Lemma \ref{lem:pnew}, we only need to consider the case such that $n_kt/T$ is an integer.
\begin{proof}
	The proof follows the proof of Lemma \ref{lem:p} closely. We provide the proof here for completeness.
	
	Construct $D'$ by changing the label values of $l=n_kt/T $ items from these $n_k$ samples falling in $B_k$, from $-T$ to $T$. Then the average label values in $B_k$ is denoted as $\bar{Y}_k'$ after such replacement. $\hat{V}_k$ also becomes $\hat{V}_k'$. Then from the $\epsilon$-label CDP requirement,
	\begin{eqnarray}
		\text{P}(\hat{V}_k = 1|\mathbf{X}_{1:N}, \bar{Y}_k=-t) &\overset{(a)}{\geq} & e^{-l\epsilon}\text{P}\left(\hat{V}_k'=1|\mathbf{X}_{1:N}, \bar{Y}_k'=-t+\frac{2l}{n_k}\right)\nonumber\\
		&\overset{(b)}{\geq} & e^{-l\epsilon}\text{P}\left(\hat{V}_k=1|\mathbf{X}_{1:N}, \bar{Y}_k=-t+\frac{2l}{n_k}\right) \nonumber\\
		&\geq & e^{-n_kt\epsilon}\left[1-\text{P}\left(\hat{V}_k=-1|\mathbf{X}_{1:N}, \bar{Y}_k=-t+\frac{2l}{n_k}\right)\right] \nonumber\\
		&\geq &\frac{1}{2}\left[1-\text{P}\left(\hat{V}_k=-1|\mathbf{X}_{1:N}, \bar{Y}_k=t\right)\right].
		\label{eq:lb1new}
	\end{eqnarray}
Here (a) uses the group privacy property. The Hamming distance between $D$ and $D'$ is $l$, thus the ratio of probability between $D$ and $D'$ is within $[e^{-l\epsilon}, e^{l\epsilon}]$. (b) holds because the algorithm does not change after changing $D$ to $D'$. Similarly,
	\begin{eqnarray}
		\text{P}(\hat{V}_k = -1|\mathbf{X}_{1:N}, \bar{Y}_k=t)\geq \frac{1}{2}\left[1-\text{P}\left(\hat{V}_k=1|\mathbf{X}_{1:N}, \bar{Y}_k=-t\right)\right].
		\label{eq:lb2new}	
	\end{eqnarray}
	Then \eqref{eq:pnew} can be shown by combining \eqref{eq:lb1new} and \eqref{eq:lb2new}.
\end{proof}

We then follow the proof of Theorem \ref{thm:clslb-cdp} in Appendix \ref{sec:clslb-cdp}. \eqref{eq:h} becomes 
\begin{eqnarray}
	h \sim \left(\frac{\epsilon N}{T}\right)^{-\frac{1}{d+\beta}}.
\end{eqnarray}
In \eqref{eq:Ydist}, note that $\text{P}(Y=T)\geq 0$ and $\text{P}(Y=-T)\geq 0$. Therefore $M_p/T^p\geq \eta_\mathbf{v}(\mathbf{x})/T$. This requires $h^\beta T^{p-1}\leq M_p$. Let $T\sim h^{-\frac{\beta}{p-1}}$, then
\begin{eqnarray}
	h \sim (\epsilon N)^{-\frac{1}{d+\beta}} h^\frac{\beta}{(d+\beta)(p-1)},
\end{eqnarray}
i.e.
\begin{eqnarray}
	h\sim (\epsilon N)^{-\frac{p-1}{p\beta+d(p-1)}}.
\end{eqnarray}
Combine with the standard minimax rate, the lower bound of regression with unbounded noise is
\begin{eqnarray}
	\underset{\mathcal{A}\in \mathcal{A}_\epsilon}{\inf}\underset{(f,\eta)\in \mathcal{F}_{reg2}}{\sup} (R-R^*)\gtrsim  N^{-\frac{2\beta}{2\beta+d}}+(\epsilon N)^{-\frac{2\beta(p-1)}{p\beta+d(p-1)}}.
\end{eqnarray}
\section{Proof of Theorem \ref{thm:tailubcdp}}\label{sec:tailubcdp}
\emph{1) Analysis of bias.} Note that Lemma \ref{lem:clipbias} still holds here. Moreover, recall \eqref{eq:bias}. Therefore
\begin{eqnarray}
	|\mathbb{E}[\hat{\eta}_l]-\eta(\mathbf{x})|\leq |\mathbb{E}[\hat{\eta}_l-\eta_T(\mathbf{x})]|+|\eta_T(\mathbf{x})-\eta(\mathbf{x})|\leq L_dh^\beta+\frac{M_p}{p-1}T^{1-p}.
\end{eqnarray}

\emph{2) Analysis of variance.} Similar to \eqref{eq:varmeany}, it can be shown that
\begin{eqnarray}
	\Var\left[\frac{1}{n_l}\sum_{i=1}^N \mathbf{1}(\mathbf{X}_i\in B_l)Y_i\right] \lesssim \frac{1}{Nh^d}.
\end{eqnarray}
Moreover, the noise variance can be bounded by
\begin{eqnarray}
	\Var[W_l]\lesssim \frac{T^2}{N^2h^{2d}\epsilon^2}.
\end{eqnarray}
The mean squared error is then bounded by
\begin{eqnarray}
	\mathbb{E}\left[\left(\hat{\eta}(\mathbf{x})-\eta(\mathbf{x})\right)^2\right] \lesssim h^{2\beta}+T^{2(1-p)}+\frac{T^2}{N^2h^{2d}\epsilon^2}+\frac{1}{Nh^d}.
\end{eqnarray}
Let $T\sim (\epsilon Nh^d)^{1/p}$, then
\begin{eqnarray}
	R-R^*=\mathbb{E}\left[(\hat{\eta}(\mathbf{X})-\eta(\mathbf{X}))^2\right]\lesssim h^{2\beta}+\frac{1}{Nh^d}+(\epsilon Nh^d)^{-2(1-1/p)}.
	\label{eq:tailexcess}
\end{eqnarray}
To minimize \eqref{eq:tailexcess}, let
\begin{eqnarray}
	h \sim N^{-\frac{1}{2\beta+d}} + (\epsilon N)^{-\frac{p-1}{p\beta + d(p-1)}},
\end{eqnarray}
then
\begin{eqnarray}
	R-R^*\lesssim N^{-\frac{2\beta}{2\beta+d}}+(\epsilon N)^{-\frac{2\beta(p-1)}{p\beta+d(p-1)}}.
\end{eqnarray}

\section{Proof of Theorem \ref{thm:fullcls}}\label{sec:fullcls}
Let $\mathcal{D}=\{(\mathbf{X}_1,Y_1), \ldots, (\mathbf{X}_N, Y_N)\}$, and $\mathcal{D}'=\{(\mathbf{X}_1',Y_1'), \ldots, (\mathbf{X}_N', Y_N')\}$, such that $\mathcal{D}$ is adjacent to $\mathcal{D}'$, i.e. $(\mathbf{X}_i, Y_i)\neq (\mathbf{X}_i', Y_i')$ only for at most one $i$. Define $n_{lj} := \sum_{i=1}^N \mathbf{1}(\mathbf{X}_i\in B_l, Y_i=j)$, $n_{lj}' := \sum_{i=1}^N \mathbf{1}(\mathbf{X}_i'\in B_l, Y_i'=j)$. Then $|n_{lj}-n_{lj}'|\leq 1$ for all $l$ and $j$.
\begin{eqnarray}
	\text{P}\left(c_1=j_1,\ldots, c_l=j_l|\mathcal{D}\right) =\Pi_{l=1}^G \frac{\exp\left(\frac{1}{4}\epsilon n_{lj_l}\right)}{\sum_{k=1}^K \exp\left(\frac{1}{4}\epsilon n_{lk}\right)},
\end{eqnarray}
and
\begin{eqnarray}
	\text{P}\left(c_1=j_1,\ldots, c_l=j_l|\mathcal{D}'\right) =\Pi_{l=1}^G \frac{\exp\left(\frac{1}{4}\epsilon n_{lj_l}'\right)}{\sum_{k=1}^K \exp\left(\frac{1}{4}\epsilon n_{lk}'\right)}.
\end{eqnarray}	
For a fixed $l$,
\begin{eqnarray} \frac{\exp\left(\frac{1}{4}\epsilon n_{lj_l}\right)}{\sum_{k=1}^K \exp\left(\frac{1}{4}\epsilon n_{lk}\right)}\slash	\frac{\exp\left(\frac{1}{4}\epsilon n_{lj_l}'\right)}{\sum_{k=1}^K \exp\left(\frac{1}{4}\epsilon n_{lk}'\right)} &=& \exp\left(\frac{1}{4}\epsilon (n_{lj_l} - n_{lj_l}')\right)\frac{\sum_{k=1}^K \exp\left(\frac{1}{4}\epsilon n_{lk}'\right)}{\sum_{k=1}^K \exp\left(\frac{1}{4}\epsilon n_{lk}\right)}\nonumber\\
	&\leq & e^{\frac{1}{4}\epsilon}	\frac{\sum_{k=1}^K \exp\left(\frac{1}{4}\epsilon (n_{lk} + 1)\right)}{\sum_{k=1}^K \exp\left(\frac{1}{4}\epsilon n_{lk}\right)}\nonumber\\
	&\leq & e^{\frac{1}{2}\epsilon}.
\end{eqnarray}
Note that since $\mathcal{D}$ and $\mathcal{D}'$ differ in only one sample, $(n_{l1}, \ldots, n_{lk})=(n_{l1}',\ldots, n_{lk}')$ except at most two $l$. Hence
\begin{eqnarray}
	\frac{\text{P}\left(c_1=j_1,\ldots, c_l=j_l|\mathcal{D}\right)}{\text{P}\left(c_1=j_1,\ldots, c_l=j_l|\mathcal{D}'\right)}\leq e^\epsilon.
\end{eqnarray}
The performance guarantee \eqref{eq:riskfullcls} can be proved following the proof of Theorem \ref{thm:clsub-cdp} in Appendix \ref{sec:clsub-cdp}. We omit the detailed steps here.

\section{Proof of Theorem \ref{thm:fullreg}}\label{sec:fullreg}
For simplicity, in this section, we only analyze the case with bounded noise. The case with unbounded noise just follows the proof of Theorem \ref{thm:tail} in Appendix \ref{sec:tail}.

\noindent \textbf{Privacy guarantee.} To prove that \eqref{eq:fullreg} is $\epsilon$-DP, we need to bound the sensitivity first. Define
\begin{eqnarray}
	\hat{\eta}_{0l} = \frac{1}{n_l\vee n_0}\sum_{i=1}^N \mathbf{1}(\mathbf{X}_i\in B_l)Y_i.
\end{eqnarray}
Let $\mathcal{D}'=\{(\mathbf{X}_1',Y_1'), \ldots, (\mathbf{X}_N', Y_N') \}$ be a dataset adjacent to $\mathcal{D}$. Define
\begin{eqnarray}
	\hat{\eta}_{0l}' = \frac{1}{n_l\vee n_0}\sum_{i=1}^N \mathbf{1}(\mathbf{X}_i'\in B_l)Y_i'.
\end{eqnarray}
Then
\begin{eqnarray}
	|\hat{\eta}_{0l}'-\hat{\eta}_{0l}| &=& \left|\frac{1}{n_l\vee n_0}\sum_{i=1}^N \mathbf{1}(\mathbf{X}_i\in B_l) Y_i- \frac{1}{n_l'\vee n_0}\sum_{i=1}^N \mathbf{1}(\mathbf{X}_i'\in B_l) Y_i\right|\nonumber\\
	&\leq & \left|\frac{1}{n_l\vee n_0}\left(\sum_{i=1}^N \mathbf{1}(\mathbf{X}_i\in B_l)Y_i - \sum_{i=1}^N \mathbf{1}(\mathbf{X}_i'\in B_l)Y_i'\right)\right|\nonumber\\&&\hspace{4mm}+ \left|\left(\frac{1}{n_l\vee n_0} - \frac{1}{n_l'\vee n_0}\right) \sum_{i=1}^N \mathbf{1}(\mathbf{X}_i'\in B_l)Y_i'\right|\nonumber\\
	&\leq & \frac{2T}{n_0} + \frac{1}{(n_l\vee n_0)(n_l'\vee n_0)}n_l'T\nonumber\\
	&\leq & \frac{3T}{n_0}.
\end{eqnarray}
Note that $\hat{\eta}_{0l}\neq\hat{\eta}_{0l}$ for at most two $l$. Therefore, let $\eta_0=(\eta_{01}, \ldots, \eta_{0G})$, and $\eta_0'=(\eta_{01}',\ldots, \eta_{0G}')$, then
\begin{eqnarray}
	\norm{\hat{\eta}_0'-\hat{\eta}_0}_1\leq \frac{6T}{n_0}.
\end{eqnarray}
Therefore, it suffices to let $W_l\sim \Lap(6T/(n_0\epsilon))$.

\noindent \textbf{Analysis of bias.} For convenience, define
\begin{eqnarray}
	\eta_l=\frac{1}{p(B_l)}\int \eta(\mathbf{u}) f(\mathbf{u}) d\mathbf{u} 
\end{eqnarray}
as the average value of $\eta$ in $B_l$ weighted by the pdf $f$.

If $n_l\geq n_0$, then
\begin{eqnarray}
	\mathbb{E}[\hat{\eta}_l|\mathbf{X}_{1:N}]=\mathbb{E}[Y|\mathbf{X}\in B_l] = \frac{1}{p(B_l)}\int \eta(\mathbf{u}) f(\mathbf{u}) d\mathbf{u} = \eta_l.
\end{eqnarray}
Otherwise
\begin{eqnarray}
	\mathbb{E}[\hat{\eta}_l|\mathbf{X}_{1:N}] = \frac{n_l}{n_0}\mathbb{E}[Y|\mathbf{X}\in B_l]= \frac{n_l}{n_0}\eta_l.
\end{eqnarray}
Therefore
\begin{eqnarray}
	|\mathbb{E}[\hat{\eta}_l] - \eta_l|&=&\frac{1}{n_0}|\eta_l|\mathbb{E}[(n_0-n_l)\mathbf{1}(n_l<n_0)]\nonumber\\
	&\leq & |\eta_l|\text{P}(n_l<n_0)\nonumber\\
	&\leq & T\exp\left(-\frac{1}{2}(1-\ln 2)n_0\right),
\end{eqnarray}
in which the last step comes from the Chernoff's inequality. From \eqref{eq:avgbias}, $|\eta_l-\eta(\mathbf{x})|\leq L_dh^\beta$ for all $\mathbf{x}\in B_l$. Therefore
\begin{eqnarray}
	|\mathbb{E}[\hat{\eta}_l]-\eta(\mathbf{x})|\leq L_dh^\beta + T\exp\left(-\frac{1}{2}(1-\ln 2)n_0\right).
\end{eqnarray}

\noindent \textbf{Analysis of variance.} We bound the first term of the right hand side of \eqref{eq:fullreg} first. 
\begin{eqnarray}
	&&\Var\left[\frac{1}{n_l\vee n_0}\sum_{i=1}^N \mathbf{1}(\mathbf{X}_i\in B_l)Y_i\right] \nonumber\\
	&&= \mathbb{E}\left[	\Var\left[\frac{1}{n_l\vee n_0}\sum_{i=1}^N \mathbf{1}(\mathbf{X}_i\in B_l)Y_i|\mathbf{X}_{1:N}\right]\right]+\Var\left[\mathbb{E}\left[\frac{1}{n_l\vee n_0}\sum_{i=1}^N \mathbf{1}(\mathbf{X}_i\in B_l)Y_i|\mathbf{X}_{1:N} \right]\right].\hspace{-1cm}\nonumber\\
	\label{eq:vardcp}
\end{eqnarray}
For the first term in \eqref{eq:vardcp},
\begin{eqnarray}
	\Var\left[\frac{1}{n_l\vee n_0}\sum_{i=1}^N \mathbf{1}(\mathbf{X}_i\in B_l)Y_i|\mathbf{X}_{1:N}\right] = \frac{n_l}{(n_l\vee n_0)^2}\Var[Y|\mathbf{X}\in B_l]\leq \frac{1}{n_0}.
\end{eqnarray}
For the second term in \eqref{eq:vardcp},
\begin{eqnarray}
	&&\Var\left[\mathbb{E}\left[\frac{1}{n_l\vee n_0}\sum_{i=1}^N \mathbf{1}(\mathbf{X}_i\in B_l)Y_i|\mathbf{X}_{1:N} \right]\right] = \Var\left[\frac{\eta_l}{n_0}n_l\mathbf{1}(n_l<n_0)\right]\nonumber\\
	&&\leq \frac{\eta_l^2}{n_0^2}\mathbb{E}[n_l^2\mathbf{1}(n_l<n_0)]\nonumber\\
	&&\leq T^2\text{P}(n_l<n_0)\nonumber\\
	&&=T^2\exp\left(-\frac{1}{2}(1-\ln 2)n_0\right).
\end{eqnarray}
Moreover, since $W_l\sim \Lap(6T/(n_0\epsilon))$,
\begin{eqnarray}
	\Var[W_l]=\frac{72T^2}{n_0^2\epsilon^2}.
\end{eqnarray}
Therefore
\begin{eqnarray}
	\Var[\hat{\eta}_l] \leq \frac{1}{n_0}+T^2\exp\left(-\frac{1}{2}(1-\ln 2)n_0\right)+\frac{72T^2}{n_0^2\epsilon^2}.
\end{eqnarray}
The remaining steps follow Appendix \ref{sec:regubcdp}. \eqref{eq:msereg} and \eqref{eq:regfinal} still hold here. The proof of Theorem \ref{thm:fullreg} is complete.

\end{document}